\documentclass[11pt,times]{article}

\usepackage[page-ref=none]{acro}

\usepackage{algorithm}
\usepackage{algorithmic}

\usepackage[sumlimits]{amsmath}
\usepackage{amssymb}
\usepackage{amsthm}

\usepackage[backend=bibtex,citestyle=numeric,bibstyle=numeric,maxcitenames=3,maxbibnames=5,firstinits=true]{biblatex}

\DeclareNameAlias{default}{last-first}

\usepackage{color}
\usepackage{diagrams}

\usepackage{enumitem}

\usepackage[mathscr]{eucal}
\usepackage{fancybox}
\usepackage{float}

\usepackage[text={6.5in,9.5in},centering]{geometry}

\usepackage{graphics}
\usepackage{graphicx}

\usepackage[hidelinks]{hyperref}

\usepackage{ifpdf}
\usepackage{ifthen}
\usepackage{lineno}
\usepackage{listings}
\usepackage{makecell}

\usepackage{multicol}
\usepackage{multirow}

\usepackage{pgfplots}

\usepackage{stackengine}

\usepackage{stmaryrd}
\usepackage{subfig}

\usepackage{tikz}
\usetikzlibrary{arrows}
\usetikzlibrary{fadings}

\usepackage{txfonts}    

\usepackage{todonotes}

\usepackage{url}
\usepackage{varioref}
\usepackage{xcolor}

\DeclareAcronym{AMP}{
  short = AMP ,
  long  = asymmetric multicore processor
}

\DeclareAcronym{BVP}{
  short = BVP ,
  long  = boundary value problem
}

\DeclareAcronym{CLEVER}{
  short = \textsc{Clever} ,
  long  = Cross Lipschitz Extreme Value for nEtwork Robustness
}

\DeclareAcronym{CPPO}{
  short = cppo ,
  long  = complete pointed partial order
}

\DeclareAcronym{CPS}{
  short = CPS ,
  long  = cyber-physical system
}

\DeclareAcronym{DAG}{
  short = DAG ,
  long  = directed acyclic graph
}

\DeclareAcronym{DBF}{
  short = DBF ,
  long  = Digital Beam Former
}

\DeclareAcronym{DCPO}{
  short = dcpo ,
  long  = directed-complete partial order
}

\DeclareAcronym{DSL}{
  short = DSL ,
  long  = domain-specific language
}

\DeclareAcronym{EDSL}{
  short = EDSL ,
  long  = embedded domain-specific language
}

\DeclareAcronym{FBP}{
  short = FBP ,
  long  = free boundary problem
}

\DeclareAcronym{FivePs}{
  short = 5Ps ,
  long  = five Ps
}

\DeclareAcronym{HS}{
  short = HS ,
  long  = hybrid system
}

\DeclareAcronym{ISA}{
  short = ISA ,
  long  = instruction set architecture
}

\DeclareAcronym{IVP}{
  short = IVP ,
  long  = initial value problem
}

\DeclareAcronym{ODE}{
  short = ODE ,
  long  = ordinary differential equation
}

\DeclareAcronym{PC}{
  short = PC ,
  long  = Pulse Compression
}

\DeclareAcronym{PCROP}{
  short = PC-ROP ,
  long  = PDE-constrained rearrangement optimization problem
}

\DeclareAcronym{PDCPO}{
  short = pointed dcpo ,
  long  = pointed directed-complete partial order
}

\DeclareAcronym{PDE}{
  short = PDE ,
  long  = partial differential equation
}

\DeclareAcronym{PI}{
  short = PI ,
  long  = principal investigator
}

\DeclareAcronym{POSET}{
  short = poset ,
  long  = partially ordered set
}

\DeclareAcronym{RE}{
  short = r.~e. ,
  long  = recursively enumerable
}

\DeclareAcronym{RNN}{
  short = RNN ,
  long  = recurrent neural network
}

\DeclareAcronym{ROP}{
  short = ROP ,
  long  = rearrangement optimization problem
}

\DeclareAcronym{SAC}{
  short = \textsc{Sac} ,
  long  = Single Assignment C
}

\DeclareAcronym{SBSE}{
  short = SBSE ,
  long  = search-based software engineerign
}

\DeclareAcronym{SOP}{
  short = SOP ,
  long  = shape optimization problem
}

\DeclareAcronym{TTE}{
  short = TTE ,
  long  = Type-II Theory of Effectivity
}

\newcommand{\defeq}{\coloneqq} 

\newcommand{\State}{\mathbb{S}}

\newcommand{\cCPPO}[1]{\mathbb{C}#1} 

\newcommand{\kCPPO}[1]{\mathbb{K}#1} 

\newcommand{\id}{\text{id}}

\newcommand{\cl}[1]{\overline{#1}}

\newcommand{\cpo}[1][s]{\ifthenelse{\equal{#1}{s}}{complete partial
order}{Complete partial order}}

\newcommand{\eg}{e.\,g.}
\newcommand{\ie}{i.\,e.}

\newcommand{\absn}[1]{\lvert\thinspace {#1} \thinspace\rvert}

\newcommand{\bra}[1]{#1^\circledcirc}

\newcommand{\clarke}[1]{\ensuremath{\partial #1}} 

\newcommand{\fpoint}[1][s]{\ifthenelse{\equal{#1}{c}}{Floating-point}{floating-point}}

\newcommand{\hbClos}[1]{\ensuremath{#1^\Box}} 

\newcommand{\interiorOf}[1]{\ensuremath{{#1}^{\circ}}}

\newcommand{\intvaldom}[1][{\Rinf}]{\ensuremath{{\mathbb{I}}#1}}
\newcommand{\intvalfun}[1][f]{\ensuremath{{\boldsymbol{I}}#1}}

\newcommand{\kc}[1][s]{\ifthenelse{\equal{#1}{s}}{Kolmogorov
complexity}{Kolmogorov complexities}} 

\newcommand{\lep}[1][x]{\ensuremath{{\underline{#1}}}} 
\newcommand{\uep}[1][x]{\ensuremath{{\overline{#1}}}} 

\newcommand{\lift}{_\bot} 

\newcommand{\glb}{\ensuremath{\bigsqcap}} 
\newcommand{\lub}{\ensuremath{\bigsqcup}} 

\newcommand{\ptime}[1][s]{\ifthenelse{\equal{#1}{c}}{Polynomial-time}{polynomial-time}}
\newcommand{\pspace}[1][s]{\ifthenelse{\equal{#1}{c}}{Polynomial-space}{polynomial-space}}

\newcommand{\me}{\mathrm{e}} 

\newcommand{\N}{\ensuremath{\mathbb{N}}}

\newcommand{\NFL}[1][c]{\ifthenelse{\equal{#1}{s}}{no free lunch}{No
Free Lunch}}

\newcommand{\norm}[1]{\ensuremath{{\left\lVert\thinspace {#1}
        \thinspace \right\rVert}}} 

\newcommand{\Q}{\ensuremath{\mathbb{Q}}}

\newcommand{\R}{\ensuremath{\mathbb{R}}}

\newcommand{\Scott}[1]{\ensuremath{\Sigma ({#1})}} 

\newcommand{\sequence}[5][i]{\ifthenelse{\equal{#1}{i}}{\ensuremath{\left< {#2}_{#3}, {#2}_{#4},
\ldots,{#2}_{#5}, \ldots \right>}}{\ensuremath{\left< {#2}_{#3}, {#2}_{#4},
\ldots,{#2}_{#5} \right>}}}

\newcommand{\set}[1]{\ensuremath{\left\{{#1}\right\}}}
\newcommand{\setbarTall}[2]{\ensuremath{\left\{{#1}\;\vrule\;
{#2}\right\}}} 
\newcommand{\setbarNormal}[2]{\ensuremath{\{{#1} \mid
{#2}\}}}

\newcommand{\UpC}[1]{\uparrow\negthickspace{#1}}

\newcommand{\wayaboves}[1]{\ensuremath{\twoheaduparrow {#1}}}
\newcommand{\waybelows}[1]{\ensuremath{\twoheaddownarrow {#1}}}

\theoremstyle{plain} 

\newtheorem{definition}{Definition}[section]

\newtheorem{corollary}[definition]{Corollary}
\newtheorem{example}[definition]{Example}
\newtheorem{lemma}[definition]{Lemma}
\newtheorem{notation}[definition]{Notation}
\newtheorem{proposition}[definition]{Proposition}
\newtheorem{remark}[definition]{Remark}
\newtheorem{theorem}[definition]{Theorem}

\newcounter{defenumalph}

\newcounter{defenum}

\floatstyle{ruled} 
\restylefloat{figure}
\restylefloat{table}

\newcounter{saveeqn}

\newcommand{\arrayoptions}[2]{\setlength{\arraycolsep}{#1}\renewcommand{\arraystretch}{#2}}

\DeclareMathOperator{\convHull}{co} 

\DeclareMathOperator{\maxS}{max_S} 

\DeclareMathOperator{\maxV}{max_V} 

\DeclareMathOperator{\ReLU}{\mathrm{ReLU}}

\title{\textbf{A Domain-Theoretic Framework for Robustness Analysis of Neural
  Networks}}

\author{Can Zhou \thanks{Department of Computer Science, University of
    Oxford, Oxford, United Kingdom,
    \href{mailto:Can.Zhou@cs.ox.ac.uk}{Can.Zhou@cs.ox.ac.uk}} \and
  Razin A. Shaikh\thanks{Department of Computer Science, University of
    Oxford, Oxford, United Kingdom, and Quantinuum Ltd., Oxford,
    United Kingdom,
    \href{mailto:Razin.Shaikh@cs.ox.ac.uk}{Razin.Shaikh@cs.ox.ac.uk}}
  \and Yiran Li \thanks{School of Computer Science, University of
    Nottingham Ningbo China,
    \href{mailto:Yiran.Li@nottingham.edu.cn}{Yiran.Li@nottingham.edu.cn}}
  \and Amin Farjudian\thanks{School of Computer Science, University of
    Nottingham Ningbo China,
    \href{mailto:Amin.Farjudian@gmail.com}{Amin.Farjudian@gmail.com}
    (\textbf{Corresponding author})}}

\date{}

\begin{document}

\maketitle

\begin{abstract}

  A domain-theoretic framework is presented for validated robustness
  analysis of neural networks. First, global robustness of a general
  class of networks is analyzed. Then, using the fact that Edalat's
  domain-theoretic $L$-derivative coincides with Clarke's generalized
  gradient, the framework is extended for attack-agnostic local
  robustness analysis. The proposed framework is ideal for designing
  algorithms which are correct by construction. This claim is
  exemplified by developing a validated algorithm for estimation of
  Lipschitz constant of feedforward regressors. The completeness of
  the algorithm is proved over differentiable networks, and also over
  general position $\ReLU$ networks. Computability results are
  obtained within the framework of effectively given domains. Using
  the proposed domain model, differentiable and non-differentiable
  networks can be analyzed uniformly. The validated algorithm is
  implemented using arbitrary-precision interval arithmetic, and the
  results of some experiments are presented. The software
  implementation is truly validated, as it handles floating-point
  errors as well.

\end{abstract}

\noindent
\textbf{Keywords:} domain theory, neural network, robustness,
Lipschitz constant, Clarke-gradient

\noindent
\textbf{MSC classes:} 06B35, 68Q55, 49J52, 68T37


\section{Introduction}
\label{sec:intro}

A system is said to be robust with respect to perturbations of a set
of parameters, if small changes to those parameters do not result in
significant changes in the behavior of the system. Robustness is a
core requirement of safety-critical systems, including neural networks
deployed in high-stakes
applications~\parencite{Dietterich:Robust_AI:2017,Heaven:AI_Easy_to_Fool:2019}. As
a result, several methods have been suggested for \emph{measuring}
robustness, {\eg}, through adversarial
attacks~\parencite{Goodfellow:Adversarial_Examples:2015,Carlini_Wagner:Towards_Robustness:2017}
or via attack-agnostic
methods~\parencite{Hein_Andriushchenko:Robustness:2017,Weng_et_al-CLEVER-ICLR:2018,Ko_et_al:POPQOEN:2019,Jordan_Dimakis:Exactly_NeurIPS:2020,Jordan_Dimakis:Provable_ICML:2021}. A
common and effective attack-agnostic approach involves estimation of
the local Lipschitz constant of the network. This makes sense as
larger Lipschitz constants signify more sensitivity to input values.
As such, a useful estimate of the Lipschitz constant should be a
\emph{tight upper bound}. A lower bound of the Lipschitz constant
results in false negatives---which are more serious than false
positives---and a very loose upper bound generates too many false
positives, which renders the estimate ineffective. Sound and accurate
computation of the Lipschitz constant is also essential in Lipschitz
regularization of neural
networks~\parencite{Araujo:Toeplitz:2021,Pauli_et_al:Lipschitz_Bounds:2022}.

Of particular interest is the estimation of Lipschitz constant of
non-differentiable networks, such as $\ReLU$ networks. Computing the
Lipschitz constant of $\ReLU$ networks is not just an NP-hard
problem~\parencite{Virmaux_Scaman:Lipschitz_regularity:2018}, it is
strongly
inapproximable~\parencite[Theorem~4]{Jordan_Dimakis:Exactly_NeurIPS:2020}. As
a result, methods of Lipschitz analysis fall into two \emph{disjoint}
categories:

\begin{enumerate}[label=(\arabic*)]
\item \textbf{Scalable}, but not validated, {\eg},
  \parencite{Weng_et_al-CLEVER-ICLR:2018,Chen_et_al:Semialgebraic_Lipschitz:2020,Latorre_et_al:Lipschitz:2020};

\item \textbf{Validated}, but not scalable, {\eg},
  \parencite{Hein_Andriushchenko:Robustness:2017,Jordan_Dimakis:Exactly_NeurIPS:2020,Jordan_Dimakis:Provable_ICML:2021},
  and the method of the current article.

\end{enumerate}

The validated methods---though not scalable---are essential in
providing the ground truth for verification of the scalable
methods. For instance, consider the scalable \ac{CLEVER} method
of~\textcite{Weng_et_al-CLEVER-ICLR:2018}. \textcite{Jordan_Dimakis:Exactly_NeurIPS:2020}
demonstrate that, on certain parts of the input domain, \ac{CLEVER}
outputs inaccurate results, even on small networks. The method of
\textcite{Jordan_Dimakis:Exactly_NeurIPS:2020,Jordan_Dimakis:Provable_ICML:2021},
in turn, although based on a validated theory, is implemented using
floating-point arithmetic, which is prone to inaccuracies, as will be
demonstrated in Section~\ref{sec:experiments}. It is well-known that
floating-point errors lead to vulnerability in safety-critical neural
networks~\parencite{Jia_Rinard:Exploiting_FP:SAS:2021,Zombori:Fooling:ICLR:2021}.

Based on the theory of continuous domains, we develop a framework for
robustness analysis of neural networks, within which we design an
algorithm for estimation of local Lipschitz constant of feedforward
regressors. We prove the soundness of the algorithm. Moreover, we
prove completeness of the algorithm over differentiable networks, and
also over general position $\ReLU$ networks.

To avoid the pitfalls of floating-point errors, we implement our
algorithm using arbitrary-precision interval arithmetic (with outward
rounding) which guarantees soundness of the implementation as well. As
such, our method is \emph{fully validated}, that is, not only does it
follow a validated theory, but also it has a validated
implementation.

\subsection{Related work}

A common approach to robustness analysis of neural networks is based
on estimation of Lipschitz constants. A mathematical justification of
this reduction is provided
in~\parencite[Theorem~2.1]{Hein_Andriushchenko:Robustness:2017}
and~\parencite[Theorem~3.2]{Weng_et_al-CLEVER-ICLR:2018}. The global
Lipschitz constant of a network may be estimated by multiplying the
Lipschitz constants of individual layers. This approach, taken
by~\textcite{Szegedy_et_al:Intriguing:2014}, is scalable but provides
loose bounds. In \acs{CLEVER}~\parencite{Weng_et_al-CLEVER-ICLR:2018}, a
statistical approach is adopted using Extreme Value
Theory~\parencite{deHaan_Ferreira:Extreme_Value_Theory:Book:2006}. While
\ac{CLEVER} is scalable, it may underestimate the Lipschitz constant
significantly~\parencite{Jordan_Dimakis:Exactly_NeurIPS:2020}. A sound
approach to Lipschitz estimation was introduced by
\textcite{Hein_Andriushchenko:Robustness:2017}, but the method applies
only to differentiable shallow networks. Semidefinite programming has
been used by \textcite{Fazlyab_et_al:Efficient_Accurate_Lipschitz:2019}
and \textcite{Hashemi_et_al:Certifying_Incremental_Quadratic:2021} to
obtain tight upper bounds on the Lipschitz constant of a network.

Of particular relevance to our work are methods that are based on
abstract
interpretation~\parencite{Cousot_Cousot:Abstract_Interpretation:1977}. The
study of abstract interpretation in a domain-theoretic framework goes
back (at least)
to~\parencite{Burn_et_al:Strictness_analysis:1986,Abramsky:Abstract_Interpretation_Kan:1990}. In
the current context, the main idea behind abstract interpretation is
to enclose values, and then follow how these enclosures propagate
through the layers of a given network. Various types of enclosures may
be used, {\eg}, hyperboxes, zonotopes, or polytopes, in increasing
order of accuracy, and decreasing order of
efficiency~\parencite{Gehr_et_al:AI2:2018,Mirman_et_al:Differentiable_Abstract_Interpretation:2018,Singh_et_al:Abstract_Domain_NN:2019}. Theoretically,
methods that are based on interval
analysis~\parencite{Wang_et_al:symbolic_intervals:2018,Prabhakar_Rahimi:Abstraction_Based_NNs:2019,Lee_et_al:Lipschitz_Certifiable:2020}
or convex outer
polytopes~\parencite{Wong_Kolter:Provable_Defenses:2018,Sotoudeh_Thakur:Abstract_NNs:2020}
are also subsumed by the framework of abstract interpretation. These
methods, however, are applied over the network itself, rather than its
gradient.

In contrast,
in~\parencite{Chaudhuri_et_al:Proving_Programs_Robust:2011,Chaudhuri_et_al:Continuity_Robustness:2012},
abstract interpretations were used for estimating Lipschitz constants
of programs with control flow, which include neural networks as
special
cases. \textcite{Jordan_Dimakis:Exactly_NeurIPS:2020,Jordan_Dimakis:Provable_ICML:2021}
have used Clarke's generalized
gradient~\parencite{Clarke:Opt_Non_Smooth_Analysis-Book:1990} to
analyze non-differentiable ({\eg}, $\ReLU$) networks. Their
overapproximation of the Clarke-gradient also follows the spirit of
abstract interpretation. Based on dual numbers, sound automatic
differentiation, and Clarke Jacobian,
\textcite{Laurel_et_al:Dual_Number:2022} have improved upon the
results of~\textcite{Chaudhuri_et_al:Proving_Programs_Robust:2011}
and~\textcite{Jordan_Dimakis:Exactly_NeurIPS:2020,Jordan_Dimakis:Provable_ICML:2021}.

Domains and neural networks serve different purposes, and there has
been little interaction between the two. One notable exception is the
use of probabilistic power domains for analysis of forgetful Hopfield
networks~\parencite{Edalat:Domains_NN:1995}. \textcite{KonecnyFarjudian2010:semantics_of_query_driven,KonecnyFarjudian2010::compositional_semantics_of_dataflow}
also present a domain-theoretic framework for analysis of a general
class of networks which communicate real values. Their focus is on
certain distributed systems which include neural networks as a special
case, but they do not discuss the issue of robustness. Domain theory,
however, has been used in robustness analysis, but not in the context
of neural networks. Specifically,
\textcite{Moggi_Farjudian_Duracz_Taha:Reachability_Hybrid:2018} have
developed a framework for robustness analysis, with an emphasis on
hybrid systems. Their work provides an important part of the
foundation of the current paper. The results
of~\parencite{Moggi_Farjudian_Duracz_Taha:Reachability_Hybrid:2018}
were developed further
in~\parencite{Moggi_Farjudian_Taha:System_Analysis_and_Robustness:2019,Moggi_Farjudian_Taha:System_Analysis_and_Robustness:ICTCS:2019,Farjudian_Moggi:Robustness_Scott_Continuity_Computability:2022:arxiv},
in the context of abstract robustness analysis.

As we will analyze maximization of interval functions in our
domain-theoretic framework, we also point out the functional algorithm
developed by \textcite{Simpson:lazy_functional_algorithms:1998} for
maximization. The focus of
\textcite{Simpson:lazy_functional_algorithms:1998} is the functional
style of the algorithm, although he uses a domain semantic model. The
domain model used in \parencite{Simpson:lazy_functional_algorithms:1998}
is, however, algebraic, which is not suitable for our framework. Our
models are all non-algebraic continuous domains.

\subsection{Contributions}

At a theoretical level, the main contribution of the current work is a
domain-theoretic framework, which provides a unified foundation for
validated methods of robustness and security analysis of neural
networks, including interval method
of~\textcite{Wang_et_al:symbolic_intervals:2018}, Clarke-gradient
methods
of~\textcite{Jordan_Dimakis:Exactly_NeurIPS:2020,Jordan_Dimakis:Provable_ICML:2021},
\textcite{Bhowmick_et_al:LipBaB:2021},
and~\textcite{Laurel_et_al:Dual_Number:2022}, and those based on
abstract interpretation,
{\eg},~\parencite{Gehr_et_al:AI2:2018,Singh_et_al:Abstract_Domain_NN:2019}. In
terms of results, the main contributions are as follows:

\begin{description}
\item[Global robustness analysis:] We prove that feedforward
  classifiers always have tightest robust \emph{approximations}
  (Theorem~\ref{thm:classifiers_tightest_robust}), and regressors always
  have robust domain-theoretic \emph{extensions}
  (Theorem~\ref{thm:regressor_interval_robust_extension}).

\item[Local robustness analysis:] Using Edalat's domain-theoretic
  $L$-derivative---which coincides with Clarke-gradient---we develop
  an algorithm for computing the Lipschitz constant of feedforward
  networks. In particular, we prove \emph{soundness and completeness}
  of the algorithm. We also prove \emph{computability} of our method
  within the framework of effectively given domains. These results are
  presented in Lemma~\ref{lemma:soundness_maximization} and
  Theorems~\ref{thm:completeness}, \ref{thm:computability},
  \ref{thm:gen_pos_ReLU} and \ref{thm:differentiable_networks}. To the
  best of our knowledge, completeness and computability results are
  missing in the previous literature on the subject.

\item[Fully validated method:] As the domains involved are effectively
  given, our method has a simple and direct implementation using
  arbitrary-precision interval arithmetic, which is correct by
  construction. Furthermore, following the completeness theorems, we
  can obtain the results to within any given degree of accuracy, for a
  broad class of networks. 

\end{description}

\subsection{Structure of the paper}

The rest of this article is structured as follows: The preliminaries
and technical background are presented in
Section~\ref{sec:preliminaries}. Section~\ref{sec:global_robustness_analysis}
contains the results on global robustness analysis. In
Section~\ref{sec:validated_lipschitz_constant}, we present the main
results on validated local robustness
analysis. Section~\ref{sec:experiments} contains the account of some
relevant experiments. Section~\ref{sec:concluding_remarks} contains
the concluding remarks.

\section{Preliminaries}
\label{sec:preliminaries}

In this section, we present the technical background which will be
used later in the paper.

\subsection{Metric Structure}
\label{sec:metric_structure}

We must first specify the metric used in measuring perturbations of
the input, and the norm used on the relevant gradients. For every
$n \in \N$ and $p \in [1,\infty]$, let $\ell_p^n$ denote the normed
space $(\R^n, \norm{.}^{(n)}_p)$, in which the norm $\norm{.}^{(n)}_p$
is defined for all $x = (x_1, \ldots, x_n) \in \R^n$ by
$\norm{x}^{(n)}_p \defeq \left( \sum_{i=1}^{n} \absn{x_i}^p
\right)^{1/p}$, if $ p \in [1, \infty)$, and
$\norm{x}^{(n)}_p \defeq \max \setbarNormal{\absn{x_i}}{ 1 \leq i \leq
  n}$, if $p = \infty$. We let
$d_p^{(n)} : \R^n \times \R^n \to [0,\infty)$ denote the metric
induced by $\norm{.}^{(n)}_p$, {\ie},
$d_p^{(n)}(x,y) = \norm{x-y}^{(n)}_p$. When $n$ is clear from the
context, we drop the superscript and write $\norm{.}_p$ and $d_p$. For
every $p \in [1,\infty]$, define the conjugate $p' \in [1,\infty]$ as
follows:
\begin{equation}
  \label{eq:p_prime_conj}
  p' \defeq \left\{
      \begin{array}{ll}
        \infty, & \text{if } p = 1,\\
        p/(p-1), & \text{if } p \in (1,\infty), \\
        1, & \text{if } p = \infty.\\        
      \end{array}
      \right.
\end{equation}
\noindent
In our discussion, whenever the perturbation of the input is measured
using $\norm{.}^{(n)}_p$-norm, then the norm of the gradient must be
taken to be $\norm{.}^{(n)}_{p'}$. Furthermore, as the spaces
$\ell^n_p$ are finite-dimensional, the topology on these spaces does
not depend on $p$, a fact which is not true in infinite-dimensional
$\ell_p$ spaces. Thus, in the current article, whenever we mention a
closed, open, or compact subset of $\R^n$, the topology is assumed to
be that of $\ell^n_\infty$.

\begin{remark}
  \label{rem:func_analysis_dual_Banach}
  The readers familiar with functional analysis may notice that the
  Banach space $\ell^n_{p'}$ is the continuous dual of the Banach
  space $\ell^n_p$. This explains the correspondence between the norm
  on the gradient and the perturbation norm. To keep the discussion
  simple, we do not present any technical background from functional
  analysis. The interested reader may refer to classical texts on the
  subject, such
  as~\parencite{Rudin:Functional_Analysis:Book:1991,Albiac_Kalton:Banach_Theory:Book:2006}
\end{remark}

\begin{remark}
\label{rem:Wasserstein_metric}
  We will only consider perturbations with respect to the
  $\norm{.}^{(n)}_p$-norm, which is the default norm considered in
  most of the related literature. We point out, however, that other
  metrics have also been suggested for quantifying perturbations, most
  notably, the Wasserstein
  metric~\parencite{Wong_Schmidt_Kolter:Wasserstein_Adversarial:2019}. In
  fact, the choice of an adequate metric remains an open
  problem~\parencite{Serban_at_al:adversarial:2020}.
\end{remark}


\subsection{Compact Subsets}

Assume that $f: \R^n \to \R$ is (the semantic model of) a regressor
neural network with $n$ input neurons. In robustness analysis, we must
consider the action of $f$ over suitable \emph{subsets} of the input
space. Specifically, we work with closed subsets of both the input and
output spaces. Let us justify this choice formally.

Let $(\State,d)$ be a metric space with
distance $d: \State \times \State \to \R$. For any
$X \subseteq \State$ and $\delta>0$, we define the open neighborhood:
\begin{equation}
  \label{eq:B_X_delta}
  B(X,\delta)\defeq\{y \in \State \mid \exists x \in X :
  d(x,y)<\delta\}.
\end{equation}
The set $B(X,\delta)$ is indeed open, because it is the union of open
balls $B(\set{s},\delta)$ with $s \in X$. Intuitively, $B(X,\delta)$
is the set of points in $X$ with imprecision less than $\delta$. For
all $X \subseteq \State$ and $\delta>0$, we have
$X\subseteq\cl{X}\subseteq B(X,\delta)=B(\cl{X},\delta)$, in which
$\cl{X}$ denotes the closure of $X$. Thus, \emph{in the presence of
  imprecision, $X$ and $\cl{X}$ are indistinguishable}. Furthermore,
it is straightforward to verify that
$\cl{X}=\bigcap_{\delta>0}B(X,\delta)=\bigcap_{\delta>0}\cl{B(X,\delta)}$. Intuitively,
$\cl{B(X,\delta)}$ is the set of points in $X$ with imprecision less
than or equal to $\delta$. Thus, \emph{the closure $\cl{X}$ is the set
  of points that are in $X$ with arbitrarily small imprecision}.

In practice, the input space of a network is a bounded region such as
$\State = [-M,M]^n$. As $[-M,M]^n$ is compact, and closed subsets of
compact spaces are compact, in this paper, we will be working with
compact subsets. A more detailed discussion of the choice of closed
subsets, in the broader context of robustness analysis, may be found
in~\parencite{Moggi_Farjudian_Taha:System_Analysis_and_Robustness:ICTCS:2019}.

\subsection{Domain Theory}
\label{subsec:domains_theory}

Domain theory has its roots in topological
algebra~\parencite{Compendium:Book:1980,Keimel:Domain_Ramifications_Interactions:2017}, and it
has enriched computer science with powerful methods from order theory,
algebra, and topology. Domains gained prominence when they were
introduced as a mathematical model for lambda calculus
by~\textcite{Scott:Outline:1970}. We have already mentioned that our
domain-theoretic framework provides a unified foundation for validated
methods of robustness and security analysis of neural networks. Domain
theory, however, provides more than just a unifying framework:

\begin{enumerate}[label=(\arabic*)]

\item Domain-theoretic models are ideal for designing algorithms which
  are \emph{correct by construction}, as will be demonstrated in
  Section~\ref{sec:validated_lipschitz_constant}. 

\item Domains provide a denotational semantics for computable analysis
  according to
  \ac{TTE}~\parencite[Theorem~9.5.2]{Weihrauch2000:book}. Our
  computability results (Theorems~\ref{thm:computability},
  \ref{thm:gen_pos_ReLU}, and \ref{thm:differentiable_networks}) are
  obtained within the framework of effectively given
  domains~\parencite{Smyth:effectively_given_domains:1977}. As a
  consequence, our constructs can be implemented directly on digital
  computers.
\end{enumerate}

We present a brief reminder of the concepts and notations that will be
needed later. The interested reader may refer
to~\parencite{Compendium:Book:1980,AbramskyJung94-DT,Gierz-ContinuousLattices-2003,Goubault-Larrecq:Non_Hausdorff_topology:2013}
for more on domains in general, and refer
to~\parencite{Escardo96-tcs,Edalat:Domains_Physics:1997} for the
interval domain, in particular. A succinct and informative survey of
the theory may be found
in~\parencite{Keimel:Domain_Ramifications_Interactions:2017}.

For any subset $X$ of a \acf{POSET} $(D, \sqsubseteq)$, we write
$\bigsqcup X$ and $\bigsqcap X$ to denote the least upper bound, and
the greatest lower bound, of $X$, respectively, whenever they
exist. In our discussion, $x \sqsubseteq y$ may be interpreted as
`\emph{$y$ contains more information than $x$}'.  A subset
$X \subseteq D$ is said to be \emph{directed} if $X\neq \emptyset$ and
every two elements of $X$ have an upper bound in $X$, {\ie},
$\forall x, y \in X: \exists z \in X: (x \sqsubseteq z) \wedge (y
\sqsubseteq z)$. An $\omega$-chain---{\ie}, a set of the form
$x_0 \sqsubseteq x_1 \sqsubseteq \cdots \sqsubseteq x_n \sqsubseteq
\cdots$---is an important special case of a directed set. A \ac{POSET}
$(D, \sqsubseteq)$ is said to be a \emph{\ac{DCPO}} if every directed
subset $X \subseteq D$ has a least upper bound in $D$. The \ac{POSET}
$(D, \sqsubseteq)$ is said to be \emph{pointed} if it has a bottom
element, which we usually denote by $\bot_D$, or if $D$ is clear from
the context, simply by $\bot$. Directed-completeness is a basic
requirement for almost all the \acp{POSET} in our discussion.

\begin{example}
  \label{example:sigma_infinity}
  Assume that $\Sigma \defeq \set{0,1}$ and let $\Sigma^*$ denote the
  set of finite strings over $\Sigma$, ordered under the prefix
  relation. Then, $\Sigma^*$ is a \ac{POSET}, but it is not
  directed-complete. For instance, the set
  $\setbarNormal{0^n}{ n \in \N}$ is an $\omega$-chain, but it does
  not have the least upper bound, or any upper bound. If we let
  $\Sigma^\omega$ denote the set of countably infinite sequences over
  $\Sigma$, then the \ac{POSET}
  $\Sigma^\infty \defeq \Sigma^* \cup \Sigma^\omega$, ordered under
  prefix relation, is a \ac{DCPO}, and may be regarded as a completion
  of $\Sigma^*$. The \ac{POSET} $\Sigma^\infty$ is indeed a \ac{PDCPO}
  with the empty string as the bottom element.
\end{example}

We let $\kCPPO{\R^n\lift}$ denote the \ac{POSET} with the carrier set
$\setbarNormal{C \subseteq \R^n}{C \text{ is non-empty and compact}}
\cup \set{\R^n}$, ordered by reverse inclusion, {\ie},
$\forall X, Y \in \kCPPO{\R^n\lift}: X \sqsubseteq Y \iff Y \subseteq
X$. By further requiring the subsets to be convex, we obtain the
sub-poset $\cCPPO{\R^n\lift}$. Finally, we let $\intvaldom[\R^n\lift]$
denote the \ac{POSET} of hyperboxes of $\R^n$---{\ie}, subsets of the
form $\prod_{i=1}^n [a_i,b_i]$---ordered under reverse inclusion, with
$\R^n$ added as the bottom element. The three \acp{POSET} are \acp{PDCPO}
and $\lub X = \bigcap X$, for any directed subset $X$.

\begin{remark}
  The reverse inclusion order $X \sqsubseteq Y \iff Y \subseteq X$
  conforms to the interpretation of $X \sqsubseteq Y$ as `$Y$ contains
  more information than $X$', in that, compared with the larger set
  $X$, the subset $Y$ is a more accurate approximation of the result.
\end{remark}

A central concept in domain theory is that of the \emph{way-below}
relation. Assume that $(D, \sqsubseteq)$ is a \ac{DCPO} and let
$x,y \in D$. The element $x$ is said to be \emph{way-below}
$y$---written as $x \ll y$---iff for every directed subset $X$ of $D$,
if $y \sqsubseteq \lub X$, then there exists an element $d \in X$ such
that $x \sqsubseteq d$. Informally, $x$ is way-below $y$ iff it is
impossible to get past $y$ (by suprema of directed sets) without first
getting past $x$. We may also interpret $x \ll y$ as `$x$ is a
finitary approximation of $y$', or even as, `$x$ is a lot simpler than
$y$'~\parencite[Section~2.2.1]{AbramskyJung94-DT}.

It is easy to prove that
$\forall x, y \in D: x \ll y \implies x \sqsubseteq y$. The converse,
however, is not always true. For instance, over $\kCPPO{\R^n\lift}$
(hence, also over $\intvaldom[\R^n\lift]$ and $\cCPPO{\R^n\lift}$) we
have the following characterization:
\begin{equation}
  \label{eq:way_below_compact_sets}
  \forall K_1, K_2 \in \kCPPO{\R^n\lift}: K_1 \ll K_2 \iff K_2 \subseteq \interiorOf{K_1},
\end{equation}
in which $\interiorOf{K_1}$ denotes the interior of $K_1$.

\begin{example}
  Consider the interval domain $\intvaldom[\R^1\lift]$ and let
  $x \defeq [0,2]$ and $y \defeq [0,1]$. Although $x \sqsubseteq y$,
  we have $\interiorOf{x} = (0,2)$. Hence,
  by~\eqref{eq:way_below_compact_sets}, we have $x \not \ll y$. We can
  also prove that $x \not \ll y$ without referring
  to~\eqref{eq:way_below_compact_sets}. Take the $\omega$-chain
  $y_n \defeq [-2^{-n}, 1 + 2^{-n}]$ for $n \in \N$. We have
  $y = \lub_{n \in \N} y_n$, but
  $\forall n \in \N: x \nsqsubseteq y_n$.
\end{example}

\begin{example}
  In the poset $\Sigma^\infty$ of
  Example~\ref{example:sigma_infinity}, we have: $x \ll y \iff (x \sqsubseteq y) \wedge (x \in \Sigma^*)$.
\end{example}

For every element $x$ of a \ac{DCPO} $(D, \sqsubseteq)$, let
$\waybelows{x} \defeq \setbarNormal{a \in D}{a \ll x}$. In
domain-theoretic terms, the elements of $\waybelows{x}$ are the true
approximants of $x$. In fact, the way-below relation is also known as
the order of
approximation~\parencite[Section~2.2.1]{AbramskyJung94-DT}. A subset
$B$ of a \ac{DCPO} $(D, \sqsubseteq)$ is said to be a \emph{basis} for
$D$, if for every element $x \in D$, the set
$B_x \defeq \waybelows{x} \cap B$ is a directed subset with supremum
$x$, {\ie}, $x = \lub B_x$. A \ac{DCPO} $(D, \sqsubseteq)$ is said to
be continuous if it has a basis, and $\omega$-continuous if it has a
countable basis. We call $(D, \sqsubseteq)$ a \emph{domain} if it is a
continuous \ac{PDCPO}. The \acp{PDCPO} $\intvaldom[\R^n\lift]$,
$\kCPPO{\R^n\lift}$, and $\cCPPO{\R^n\lift}$, are all
$\omega$-continuous domains, because:

\label{page:domain}

\begin{itemize}
\item
  $B_{\intvaldom[\R^n\lift]} \defeq \set{\R^n} \cup \setbarNormal{C
    \in \intvaldom[\R^n\lift]}{C \text{ is a hyperbox with
      rational coordinates}}$ is a basis for $\intvaldom[\R^n\lift]$;

\item $B_{\kCPPO{\R^n\lift}} \defeq \set{\R^n} \cup \setbarNormal{C \in
  \kCPPO{\R^n\lift}}{C \text{ is a finite union of hyperboxes with
    rational coordinates}}$ is a basis for $\kCPPO{\R^n\lift}$;

\item
  $B_{\cCPPO{\R^n\lift}} \defeq \set{\R^n} \cup \setbarNormal{C \in
    \cCPPO{\R^n\lift}}{C \text{ is a convex polytope with rational
      coordinates}}$ is a basis for $\cCPPO{\R^n\lift}$.

\end{itemize}

\label{page:StoltenbergHansen_Tucker:mentioned}

The poset $\Sigma^\infty$ of Example~\ref{example:sigma_infinity} is
also an $\omega$-continuous domain, with $\Sigma^*$ as a countable
basis. There is a significant difference between $\Sigma^\infty$ and
the domains $\intvaldom[\R^n\lift]$, $\kCPPO{\R^n\lift}$, and
$\cCPPO{\R^n\lift}$, in that, for the elements of the basis $\Sigma^*$
of $\Sigma^\infty$, we have: $\forall x \in \Sigma^*: x \ll
x$. By~\eqref{eq:way_below_compact_sets}, this is clearly not true of
the bases $B_{\intvaldom[\R^n\lift]}$, $B_{\kCPPO{\R^n\lift}}$ and
$B_{\cCPPO{\R^n\lift}}$, or indeed, any other bases for
$\intvaldom[\R^n\lift]$, $\kCPPO{\R^n\lift}$, and
$\cCPPO{\R^n\lift}$. As such, the domain $\Sigma^\infty$ is an example
of an algebraic domain, whereas the domains that we use in the current
paper are non-algebraic. Although algebraic domains have been used in
real number
computation~\parencite{DiGianantonio:Real_Domain:IC:1996,Farjudian:Shrad:2007}
and, more broadly, representation of topological
spaces~\parencite{StoltenbergHansen_Tucker:Effective_Algebras:1995,StoltenbergHansen_Tucker:topological_algebras:1999},
non-algebraic domains have proven more suitable for computation over
continuous spaces, {\eg}, in dynamical
systems~\parencite{Edalat95:DT-fractals}, exact real number
computation~\parencite{Escardo96-tcs,Edalat:Domains_Physics:1997},
differential equation
solving~\parencite{Edalat_Pattinson2007-LMS_Picard,Edalat_Farjudian_Mohammadian_Pattinson:2nd_Order_Euler:2020:Conf},
and reachability analysis of hybrid
systems~\parencite{Edalat_Pattinson:Hybrid:2007,Moggi_Farjudian_Duracz_Taha:Reachability_Hybrid:2018},
to name a few. Hence, in this article, we will be working in the
framework of non-algebraic $\omega$-continuous domains.

Apart from the order-theoretic structure, domains also have a
topological structure. Assume that $(D, \sqsubseteq)$ is a
\ac{POSET}. A subset $O \subseteq D$ is said to be \emph{Scott open}
if it has the following properties:
  \begin{enumerate}[label=(\arabic*)]
  \item It is an upper set, that is,
    $\forall x \in O, \forall y \in D: x \sqsubseteq y \implies y \in
    O$.
  \item It is inaccessible by suprema of directed sets, that is, for
    every directed set $X \subseteq D$ for which $\lub X$ exists, if
    $\lub X \in O$ then $X \cap O \neq \emptyset$.
  \end{enumerate}

  \begin{example}
    Consider the poset $\R$ under the usual order on real
    numbers. Then, for any $a \in \R$, the set
    $\setbarNormal{x \in \R}{x > a}$ is Scott open. Note that
    $(\R, \leq)$ is not a domain, because it is not even complete. If
    we extend $\R$ to $\R_{\pm \infty} \defeq [-\infty, +\infty]$ with
    $\forall x \in \R: -\infty < x < +\infty$, then
    $(\R_{\pm \infty}, \leq)$ is an $\omega$-continuous domain (in
    fact, an $\omega$-continuous lattice) with $\Q \cup \set{-\infty}$
    as a basis, and each set of the form $(a,+\infty]$ is Scott open.
  \end{example}

  \begin{example}
    Consider the domain $\intvaldom[\R\lift]$. The set
    $\setbarNormal{[a,b]}{[a,b] \subseteq (0,1)}$ is Scott open in
    $\intvaldom[\R\lift]$. See also
    Proposition~\ref{prop:Scott_wayaboves}
    \vpageref[below]{prop:Scott_wayaboves}.
  \end{example}

The collection of all Scott open subsets of a \ac{POSET} forms a $T_0$
topology, referred to as the Scott topology. A function
$f: (D_1, \sqsubseteq_1) \to (D_2, \sqsubseteq_2)$ is said to be
Scott-continuous if it is continuous with respect to the Scott
topologies on $D_1$ and $D_2$. Scott continuity can be stated purely
in order-theoretic terms. A map
$f: (D_1, \sqsubseteq_1) \to (D_2, \sqsubseteq_2)$ between two posets
is Scott continuous if and only if it is monotonic and preserves the
suprema of directed sets, {\ie}, for every directed set
$X \subseteq D_1$ for which $\lub X$ exists, we have
$f(\lub X) = \lub
f(X)$~\parencite[Proposition~4.3.5]{Goubault-Larrecq:Non_Hausdorff_topology:2013}. When
$D_1$ and $D_2$ are $\omega$-continuous domains, then we have an even
simpler formulation:

  \begin{proposition}[{\parencite[Proposition~2.2.14]{AbramskyJung94-DT}}]
    \label{prop:Scott_w_cont}
    Assume that $f: D_1 \to D_2$ is
    a map between two $\omega$-continuous domains. Then, $f$ is Scott
    continuous iff it is monotonic and preserves the suprema of
    $\omega$-chains.
  \end{proposition}

  For every element $x$ of a \ac{DCPO} $(D, \sqsubseteq)$, let
  $\wayaboves{x} \defeq \setbarNormal{a \in D}{x \ll a}$.
  \begin{proposition}[{\parencite[Proposition~2.3.6]{AbramskyJung94-DT}}]
    \label{prop:Scott_wayaboves}
    Let $D$ be a domain with a basis $B$. Then, for each $x \in D$,
    the set $\wayaboves{x}$ is Scott open, and the collection
    ${\mathcal{O}} \defeq \setbarTall{\wayaboves{x}}{x \in B}$ forms a
    base for the Scott topology.
  \end{proposition}

  The maximal elements of $\intvaldom[\R^n\lift]$,
  $\kCPPO{\R^n\lift}$, and $\cCPPO{\R^n\lift}$ are singletons, and the
  sets of maximal elements may be identified with $\R^n$. For
  simplicity, we write $x$ to denote a maximal element $\set{x}$. As a
  corollary of Proposition~\ref{prop:Scott_wayaboves} and
  characterization~\eqref{eq:way_below_compact_sets}, we obtain:

    \begin{corollary}
    \label{cor:Scott_restrict_Euclid}
    Let ${\cal O}_S$ be the Scott topology on $\intvaldom[\R^n\lift]$,
    $\kCPPO{\R^n\lift}$, or $\cCPPO{\R^n\lift}$. Then, the restriction
    of ${\mathcal{O}}_S$ over $\R^n$ is the Euclidean topology.
  \end{corollary}

  Thus, the sets of maximal elements are homeomorphic to $\R^n$. For
  any $K \geq 0$, by restricting to $[-K,K]^n$, we obtain the
  $\omega$-continuous domains $\intvaldom[{[-K,K]^n}]$,
  $\kCPPO{[-K,K]^n}$, and $\cCPPO{[-K,K]^n}$, respectively.

  We say that a \ac{POSET} $(D,\sqsubseteq)$ is
  \emph{bounded-complete} if each bounded pair $x,y \in D$ has a
  supremum. Assume that $(X, \Omega(X))$ is a topological space, and
  let $(D, \sqsubseteq_D)$ be a bounded-complete domain. We let
  $(D, \Scott{D})$ denote the topological space with carrier set $D$
  under the Scott topology $\Scott{D}$. The space $[X \to D]$ of
  functions $f: X \to D$ which are $(\Omega(X), \Scott{D})$ continuous
  can be ordered pointwise by defining:
\begin{equation*}
  \forall f, g \in [X \to D]: \quad f \sqsubseteq g \iff \forall x \in X:
  f(x) \sqsubseteq_D g(x).
\end{equation*}
It is straightfoward to verify that the \ac{POSET}
$([X \to D], \sqsubseteq)$ is directed-complete and
$\forall x \in X: (\lub_{i \in I} f_i)(x) = \lub
\setbarNormal{f_i(x)}{i \in I}$, for any directed subset
$\setbarNormal{f_i}{i \in I}$ of $[X \to D]$. For the \ac{POSET}
$([X \to D], \sqsubseteq)$ to be continuous, however, the topological
space $(X, \Omega(X))$ must be \emph{core-compact}, as we explain
briefly. Consider the \ac{POSET} $(\Omega(X), \subseteq)$ of open
subsets of $X$ ordered under subset relation. For any topological
space $X$, this \ac{POSET} is a complete lattice, with $\emptyset$ as
the bottom element, and $X$ as the top element. Furthermore, we have:
\begin{equation*}
  \forall A \subseteq \Omega(X): \quad \lub A = \bigcup A \text{ and
  } \glb A = \interiorOf{(\bigcap A)}.
\end{equation*}
A topological space $(X, \Omega(X))$ is said to be core-compact if and
only if the lattice $(\Omega(X), \subseteq)$ is continuous. 
\begin{theorem}
  \label{thm:core_compact}
  For any topological space $(X, \Omega(X))$ and non-singleton
  bounded-complete continuous domain $(D, \sqsubseteq_D)$, the
  function space $([X \to D], \sqsubseteq)$ is a bounded-complete
  continuous domain $\iff (X, \Omega(X))$ is core-compact.
\end{theorem}

\begin{proof}
  For the ($\Leftarrow$) direction,
  see~\parencite[Proposition~2]{Erker_et_al:way_below:1998}. A proof
  of the ($\Rightarrow$) direction can also be found
  on~\parencite[pages 62 and 63]{Erker_et_al:way_below:1998}.
\end{proof}

  \begin{notation}[$X \Rightarrow D$]
    \label{notation:fun_space}
    For any core-compact topological space $(X, \Omega(X))$ and
    bounded-complete continuous domain $(D, \sqsubseteq_D)$, we denote
    the continuous domain $([X \to D], \sqsubseteq)$ by the notation
    $X \Rightarrow D$.
  \end{notation}

  All the domains that will be used in the framework developed in this
  article---including, $\intvaldom[{[-K,K]^n}]$, $\kCPPO{[-K,K]^n}$,
  and $\cCPPO{[-K,K]^n}$---are bounded-complete, and they are
  core-compact under their respective Scott topologies. Furthermore,
  the space $[-K,K]^n$ under the Euclidean topology is
  core-compact. For more on core-compact spaces, the interested reader
  may refer to, {\eg}, ~\parencite[Chapter~II]{Compendium:Book:1980}
  and~\parencite[Chapter~5]{Goubault-Larrecq:Non_Hausdorff_topology:2013}.

\label{page:interval_function}
  
  Throughout this article, by an \emph{interval function} over a set
  $X$ we mean a function $f: X \to \intvaldom[\R^n\lift]$ for some
  $n \geq 1$. For any set $X \subseteq \R^n$, we let $D^{(0)}(X)$
  denote the continuous domain of interval functions
  $X \Rightarrow \intvaldom[\R\lift]$. Each function in $D^{(0)}(X)$
  is Euclidean-Scott-continuous, {\ie}, continuous with respect to the
  Euclidean topology on $X$ and Scott topology on
  $\intvaldom[\R\lift]$. The functions in $D^{(0)}(X)$ have a useful
  characterization in terms of semicontinuity. Recall that, if $X$ is
  any topological space, then:

\begin{itemize}
\item $f: X \to \R$ is said to be upper semi-continuous at $x_0 \in X$
  iff, for every $y > f(x_0)$, there exists a neighborhood $U$ of
  $x_0$ (in the topology of $X$) such that
  $\forall x \in U: f(x) < y$.

\item $f: X \to \R$ is said to be lower semi-continuous at $x_0 \in X$
  iff, for every $y < f(x_0)$, there exists a neighborhood $U$ of
  $x_0$ such that $\forall x \in U: f(x) > y$.

\item $f: X \to \R$ is said to be upper (respectively, lower)
  semi-continuous iff it is upper (respectively, lower)
  semi-continuous at every $x_0 \in X$.

\end{itemize}

\begin{proposition}[\parencite{Edalat_Lieutier:Domain_Calculus_One_Var:MSCS:2004}]
    \label{prop:Euclidean_Scott_Cont}
    Assume that $X \subseteq \R^n$ and
    $f \equiv [\lep[f], \uep[f]] \in D^{(0)}(X)$, in which $\lep[f]$
    and $\uep[f]$ are the lower and upper bounds of the interval
    function $f$, respectively. Then, $f$ is
    Euclidean-Scott-continuous $\iff$ $\lep[f]$ is lower
    semicontinuous, and $\uep[f]$ is upper semicontinuous.
  \end{proposition}

For any compact set $K \in \kCPPO{\R^n\lift}$, we define the hyperbox
closure as:
  \begin{equation}
    \label{eq:hyper_rect_clos}
    \hbClos{K} \defeq \lub{\setbarNormal{R \in
        \intvaldom[\R^n\lift]}{K \subseteq R}}. 
  \end{equation}
  As $B_{\intvaldom[\R^n\lift]}$ is a basis for
  $\intvaldom[\R^n\lift]$, it is straightforward to show that:
  \begin{proposition}
    \label{prop:K_box_basis_formulation}
    For any compact set $K \in \kCPPO{\R^n\lift}$, we have
    $\hbClos{K} =   \lub{\setbarNormal{b \in
        B_{\intvaldom[\R^n\lift]}}{ b \ll K}}$.
  \end{proposition}

  For each $i \in \set{1, \ldots, n}$, we let $\pi_i: \R^n \to \R$
  denote projection over the $i$-th component, that is,
  $\pi_i( x_1, \ldots, x_n) \defeq x_i$. The projections are all
  continuous functions, so they map compact sets to compact
  sets. Hence, for any compact set $K \in \kCPPO{\R^n\lift}$, the set
  $\pi_i(K)$ is a compact subset of $\R$, for any given
  $i \in \set{1, \ldots, n}$. For each $i \in \set{1, \ldots, n}$, we
  define $\lep[K_i] \defeq \min \pi_i(K)$ and
  $\uep[K_i] \defeq \max \pi_i(K)$. It is straightforward to verify
  that:
  \begin{equation}
    \label{eq:K_box_prod}
    \forall K \in \kCPPO{\R^n\lift}: \quad \hbClos{K} = \prod_{i=1}^n
    [\lep[K_i], \uep[K_i]].
  \end{equation}

  \begin{lemma}
    \label{lemma:b_waybelow_K_box}
    For any hyperbox $b \in \intvaldom[\R^n\lift]$ and compact set
    $K \in \kCPPO{\R^n\lift}$, we have: $b \ll K \iff b \ll \hbClos{K}$.
  \end{lemma}

  \begin{proof}
    Clearly $b \ll \hbClos{K}$ implies $b \ll K$. For the opposite
    direction, assume that $b = \prod_{i=1}^n
    b_i$. From~\eqref{eq:way_below_compact_sets}, we deduce that
    $b \ll K$ iff
    $\forall i \in \set{1, \ldots, n}: (\lep[b_i] < \lep[K_i]) \wedge
    (\uep[K_i] < \uep[b_i])$. This, combined
    with~\eqref{eq:K_box_prod}, entails $b \ll \hbClos{K}$.
  \end{proof}

  \begin{corollary}
    \label{cor:box_map_scott_cont}
    The map
    $\hbClos{(\cdot)} : \kCPPO{\R^n\lift} \to \intvaldom[\R^n\lift]$
    is Scott-continuous.
  \end{corollary}

  \begin{proof}
    As both $\kCPPO{\R^n\lift}$ and $\intvaldom[\R^n\lift]$ are
    $\omega$-continuous, by Proposition~\ref{prop:Scott_w_cont}, it
    suffices to show that the map $\hbClos{(\cdot)}$ is monotonic and
    preserves the suprema of $\omega$-chains. Monotonicity is
    straightforward. Next, assume that $(K_i)_{i \in \N}$ is an
    $\omega$-chain in $\kCPPO{\R^n\lift}$. We must show that
    $\hbClos{\left( \lub_{i \in \N} K_i\right)} = \lub_{i \in \N}
    \hbClos{K_i}$. The $\sqsupseteq$ direction follows from
    monotonicity. To prove the $\sqsubseteq$ direction, by
    Proposition~\ref{prop:K_box_basis_formulation}, we have:
    \begin{equation*}
      \hbClos{(\lub_{i \in \N} K_i)} = \lub{\setbarTall{b \in
          B_{\intvaldom[\R^n\lift]}}{ b \ll \lub_{i \in \N} K_i}}.  
    \end{equation*}
    Take any arbitrary $b \in B_{\intvaldom[\R^n\lift]}$ satisfying
    $b \ll \lub_{i \in \N} K_i$. By
    Proposition~\ref{prop:Scott_wayaboves}, the set $\wayaboves{b}$ is
    Scott-open, which entails that
    $\exists i_0 \in \N: b \ll K_{i_0}$. By
    Lemma~\ref{lemma:b_waybelow_K_box}, we must have
    $b \ll \hbClos{K_{i_0}}$, from which we deduce that
    $b \ll \lub_{i \in \N} \hbClos{K_i}$. As $b$ was chosen
    arbitrarily, then we must have
    $\hbClos{\left( \lub_{i \in \N} K_i\right)} \sqsubseteq \lub_{i
      \in \N} \hbClos{K_i}$.
  \end{proof}

    \begin{corollary}
      \label{cor:IR_subdomain_KR}
    For every $n \geq 1$, $\intvaldom[\R^n\lift]$ is a sub-domain of
        $\kCPPO{\R^n\lift}$.
  \end{corollary}

  \begin{proof}
    Consider the map
    $\iota : \intvaldom[\R^n\lift] \to \kCPPO{\R^n\lift}$ defined by
    $\forall x \in \intvaldom[\R^n\lift]: \iota( x) \defeq x$. It is
    straightforward to prove that $\iota$ is Scott-continuous,
    $\iota \circ \hbClos{(\cdot)} \sqsubseteq
    \id_{\kCPPO{\R^n\lift}}$, and
    $\hbClos{(\cdot)} \circ \iota =
    \id_{\intvaldom[\R^n\lift]}$. Therefore,
    $(\iota, \hbClos{(\cdot)})$ form a continuous embedding-projection
    pair.
  \end{proof}

\begin{definition}[Extension, Canonical Interval Extension {$\intvalfun[f]$}, Approximation]{\ }
  \label{def:ext_canonical_ext}

  \begin{enumerate}[label=(\roman*)]

  \item \label{item:extension} A map
    $u: \kCPPO{\R^n\lift} \to \kCPPO{\R^m\lift}$ is said to be an
    extension of $f: \R^n \to \R^m$ iff
    $\forall x \in \R^n: u(\set{x}) = \set{f(x)}$.

  \item A map $u: \intvaldom[\R^n\lift] \to \intvaldom[\R^m\lift]$ is
    said to be an interval extension of
    $f: \R^n \to \kCPPO{\R^m\lift}$ iff
    $\forall x \in \R^n: u(\set{x}) = \hbClos{f(x)}$.

  \item \label{item:canonical_interval_extension} For any
    $f : \R^n \to \kCPPO{\R^m\lift}$, we define the canonical interval
    extension
    ${\intvalfun[f]} : \intvaldom[\R^n\lift] \to
    \intvaldom[\R^m\lift]$ by:
    \begin{equation}
      \label{eq:canonical_ext}
      \forall \alpha \in \intvaldom[\R^n\lift]: \quad
      {\intvalfun[f]}(\alpha) 
      \defeq \bigsqcap_{x \in \alpha} \hbClos{f(x)}.
    \end{equation}

  \item A map $u: \intvaldom[\R^n\lift] \to \intvaldom[\R^m\lift]$ is
    said to be an interval approximation of
    $f: \R^n \to \kCPPO{\R^m\lift}$ if $u \sqsubseteq \intvalfun[f]$.   
  \end{enumerate}
\end{definition}

\begin{proposition}
\label{prop:canonical_interval_extension}
  For every Euclidean-Scott-continuous
  $f : \R^n \to \kCPPO{\R^m\lift}$, the canonical interval extension
  ${\intvalfun[f]}$ defined in~\eqref{eq:canonical_ext} is the maximal
  extension of $f$ among all the extensions in the continuous domain
  $\intvaldom[\R^n\lift] \Rightarrow \intvaldom[\R^m\lift]$. In
  particular, ${\intvalfun[f]}$ is Scott-continuous.
\end{proposition}

\begin{proof}
  Given a map $f : \R^n \to \kCPPO{\R^m\lift}$, we define
  $\hbClos{f}: \R^n \to \intvaldom[\R^m\lift]$ by
  $\hbClos{f} \defeq \hbClos{(\cdot)} \circ f$, {\ie},
  $\forall x \in \R^n: \hbClos{f}(x) = \hbClos{f(x)}$. If $f$ is
  Euclidean-Scott-continuous, then, by
  Corollary~\ref{cor:box_map_scott_cont}, so is $\hbClos{f}$. It is
  straightforward to verify that a map
  $u: \intvaldom[\R^n\lift] \Rightarrow \intvaldom[\R^m\lift]$ is an
  interval approximation of $f$ if and only if it is an interval
  approximation of $\hbClos{f}$. Thus, it suffices to prove the
  proposition for the special case of
  $f: \R^n \to \intvaldom[\R^m\lift]$. This has been proved
  in~\parencite[Lemma~3.4]{Edalat_Escardo:Integ_realPCF:2000} for
  $n = m = 1$. The proof given
  in~\parencite{Edalat_Escardo:Integ_realPCF:2000}, however, is
  independent of the values of $n$ and $m$, and the main property that
  is required is that $\intvaldom[\R^m\lift]$ is a continuous
  $\glb$-semilattice, for any $m \in \N$.
\end{proof}

As the restriction of the Scott topology of $\kCPPO{\R^m\lift}$ over
$\R^m$ is the Euclidean topology, we may consider any continuous map
$f: \R^n \to \R^m$ also as a function of type
$\R^n \to \kCPPO{\R^m\lift}$, and construct its canonical interval extension
accordingly, which will be Scott-continuous.

\begin{remark}
  The distinction between extension and approximation is crucial in
  the current article.
\end{remark}


  \section{Global Robustness Analysis}
  \label{sec:global_robustness_analysis}

A domain-theoretic framework for abstract robustness analysis was
introduced
by~\textcite{Moggi_Farjudian_Duracz_Taha:Reachability_Hybrid:2018}. The
essence of the framework is in the link between robustness---a concept
defined based on perturbations---and Scott topology---a concept
defined in terms of order theory.

  Assume that $\State$ is a metric space, $\kCPPO{\State}$ is the
  \ac{PDCPO} of non-empty compact subsets of $\State$, ordered under
  reverse inclusion, with $\State$ added as the bottom element if
  $\State$ is not compact.  When $\State$ is a compact metric
  space---which is the relevant case in our discussion---the
  \ac{PDCPO} $\kCPPO{\State}$ becomes an $\omega$-continuous
  domain~\parencite[Proposition~3.4]{Edalat95:DT-fractals}. For each
  $X \subseteq \State$, let
  $\UpC{X}\defeq\setbarNormal{C \in \kCPPO{\State}}{C\subseteq
    X}$. The robust topology is defined over
  $\kCPPO{\State}$. Intuitively, we regard a collection
  $U \subseteq \kCPPO{\State}$ of compact subsets of $\State$ as
  robust open if $U$ is closed under sufficiently small
  perturbations. Formally, we say that $U \subseteq \kCPPO{\State}$ is
  \emph{robust open} iff
  $ \forall C \in U : \exists \delta > 0 : \UpC{B(C,\delta)}\subseteq
  U$, in which $B(C,\delta)$ is as defined
  in~\eqref{eq:B_X_delta}. This topology indeed captures the notion of
  robustness~\parencite[Theorem~A.2]{Moggi_Farjudian_Duracz_Taha:Reachability_Hybrid:2018}. Thus,
  we define:

\begin{definition}[Robust Map]
  Assume that $\State_1$ and $\State_2$ are compact metric spaces. We
  say that $f : \kCPPO{\State_1} \to \kCPPO{\State_2}$ is robust iff
  it is continuous with respect to the robust topologies on
  $\kCPPO{\State_1}$ and $\kCPPO{\State_2}$.
\end{definition}

  \begin{theorem}[{\parencite[Theorem
    A.4]{Moggi_Farjudian_Duracz_Taha:Reachability_Hybrid:2018}}]
    \label{thm:compact_metric}
    For any compact metric space $\State$, the Scott and robust topologies on
    $\kCPPO{\State}$ coincide.
  \end{theorem}

  Assume that $N : \State_1 \to \State_2$ is a network, and define the
  reachability map $A_N : \kCPPO{\State_1} \to \kCPPO{\State_2}$ as
  follows:

\begin{equation}
  \label{eq:A_N_reachability_map}
  \forall X \in \kCPPO{\State_1}: \quad A_N(X) \defeq
  \bigcap \setbarNormal{C \in \kCPPO{\State_2}}{N(X) \subseteq C}.
\end{equation}
\noindent
The map $A_N$ is clearly monotonic, that is,
$\forall X, Y \in \kCPPO{\State_1}: X \subseteq Y \implies A_N(X)
\subseteq A_N(Y)$. In case $A_N$ is not robust, a valid question is
whether it is possible to approximate $A_N$ with a robust map without
losing too much accuracy. For compact metric spaces, the answer is
affirmative:

  \begin{corollary}
    \label{cor:tightest_robust_approximation}
    If $\State_1$ and $\State_2$ are compact metric spaces, then any
    monotonic map $f : \kCPPO{\State_1} \to \kCPPO{\State_2}$ has a
    `tightest' robust approximation
    $\bra{f} : \kCPPO{\State_1} \to \kCPPO{\State_2}$ satisfying:
    \begin{equation}
      \label{eq:bra_formulation}
      \forall X \in \kCPPO{\State_1}: \quad \bra{f}(X) = \lub
      \setbarNormal{f(b)}{ b \ll X},
    \end{equation}
    \noindent
    which by Theorem~\ref{thm:compact_metric}, must be
    Scott-continuous. Here, by the tightest we mean the largest under the
    pointwise ordering of functions.
  \end{corollary}

  \begin{proof}
    The existence of a tightest robust approximation follows
    from~\parencite[Corollaries 4.4 and
    4.5]{Moggi_Farjudian_Duracz_Taha:Reachability_Hybrid:2018}. The
    formulation~\eqref{eq:bra_formulation} follows
    from~\parencite[Theorem~5.20]{Moggi_Farjudian_Duracz_Taha:Reachability_Hybrid:2018}.
  \end{proof}

\label{page:classifier_reduce_to_regressor}

  In a typical application, the input space of a neural network is a
  bounded region of some Euclidean space. Therefore, we consider
  feedforward neural networks with input space $[-M,M]^n$, for some
  $M>0$. For the output, we will mainly focus on the regressors of
  type $N : [-M,M]^n \to \R$. The reason is that, within the framework
  that we have adopted, robustness analysis of classifiers reduces to
  Lipschitz estimation of regressors. Assume that
  $\hat{N}: [-M,M]^n \to \set{c_1, \ldots, c_k}$ is a classifier
  network. A common architecture for such a classifier comprises a
  regressor $N \equiv (N_1, \ldots, N_k): [-M,M]^n \to \R^k$ followed
  by an $\arg \max$ operation. Thus, for a given input
  $x_0 \in [-M,M]^n$, the index of the class predicted by $\hat{N}$ is
  $i_0 = \arg \max_{1 \leq i \leq k} N_i(x_0)$. In certified local
  robustness analysis, the aim is to provide a radius $\beta > 0$ such
  that, for any point $x$ in the neighborhood of radius $\beta$ around
  $x_0$, we have $\hat{N}(x) = \hat{N}(x_0)$, {\ie}, a neighborhood
  within which no adversarial examples can be found. A lower bound for
  such a radius can be obtained by estimating the Lipschitz constants
  of the functions
  $\setbarNormal{ N_{i_0} - N_i}{ 1 \leq i \leq k, i \neq i_0}$. As
  such, robustness analysis of classifiers reduces to Lipschitz
  estimation of regressors with one output. For more details, the
  reader may refer
  to~\parencite[Section~3]{Weng_et_al-CLEVER-ICLR:2018}.

  When $N: [-M,M]^n \to \R$ is a continuous function---which is the
  relevant case in our discussion---the image of the compact set
  $[-M,M]^n$ will also be compact, which we may assume to be included
  in a compact interval $[-M',M']$, for $M' \geq 0$ large
  enough. Furthermore, the reachability map $A_N$
  of~\eqref{eq:A_N_reachability_map} satisfies
  $\forall X \in \kCPPO{[-M,M]^n}: A_N(X) = N(X)$.

  Before focusing on regressors, however, we briefly discuss some
  issues related to robustness analysis of classifiers.


  \subsection{Classifiers}
\label{subsec:classifiers}

Let $C = \set{c_1, \ldots, c_k}$ be a set of classes, with $k \geq 1$,
and assume that the function $N: [-M,M]^n \to C$ is (the semantic model
of) a classifier. As $C$ is discrete, it is natural to metrize $C$
using the discrete metric:
\begin{equation*}
  \forall x,y \in C: \quad  d(x,y) \defeq
  \left\{
    \begin{array}{ll}
      0, & \text{if } x = y,\\
      1, & \text{if } x \neq y.\\
    \end{array}
  \right.
\end{equation*}
\noindent
The main problem with a classifier of type $N : [-M,M]^n \to C$ is the
mapping of a connected space to a discrete one, where the topology on
$[-M,M]^n$ is the Euclidean topology, and the topology on $C$ is the
discrete topology:

\begin{proposition}
\label{prop:classifier_cont_constant}
A classifier $N: [-M,M]^n \to C$ is continuous iff $N$ is constant.
\end{proposition}

\begin{proof}
  The $\Leftarrow$ direction is obvious. For the $\Rightarrow$
  direction, as $C$ is a discrete space, then each singleton
  $\set{c_i}$ is an open set. If $N$ is continuous, then each inverse
  image $N^{-1}(c_i)$ must also be open. If $N$ takes more than one
  value, then $[-M,M]^n = \bigcup_{1 \leq i \leq k} N^{-1}(c_i)$ must be
  the union of at least two disjoint non-empty open sets, which is
  impossible as $[-M,M]^n$ is a connected space.
\end{proof}

Proposition~\ref{prop:classifier_cont_constant} implies that a
classifier may not be computable---in the framework of
\ac{TTE}~\parencite{Weihrauch2000:book}---unless if it is
constant. This is because, by~\parencite[Theorem~2.13]{Ko91-book}, if
$N$ is computable, then it must be continuous.

By moving to compact subsets, however, it is possible to obtain tight
robust approximations of classifiers, even in the presence of
discontinuities. As $C$ is a finite set, it is compact under the
discrete topology, and $\kCPPO{C}$ is the set of non-empty subsets of
$C$, ordered under reverse inclusion.

\begin{theorem}
  \label{thm:classifiers_tightest_robust}
  Assume that $A_N : \kCPPO{[-M,M]^n} \to \kCPPO{C}$ is the reachability
  map for a classifier $N: [-M,M]^n \to C$, defined
  in~\eqref{eq:A_N_reachability_map}. Then, $A_N$ has a tightest
  robust approximation $\bra{A_N} : \kCPPO{[-M,M]^n} \to \kCPPO{C}$
  which is also Scott-continuous.
\end{theorem}

\begin{proof}
  Since both $[-M,M]^n$ and $C$ are compact, the result follows from
  Corollary~\ref{cor:tightest_robust_approximation}.
\end{proof}

\begin{example}
  \label{example:classifier_bra}
  Assume that the classifier $N: [-1,1] \to \set{c_1, c_2}$ satisfies:
  \begin{equation*}
    \forall x \in [-1,1]: \quad N(x) = \left\{
      \begin{array}{ll}
        c_1, & \text{if } x \in [-1,0],\\
        c_2, & \text{if } x \in (0,1].\\
      \end{array}
      \right.
    \end{equation*}
    Then, according to~\eqref{eq:A_N_reachability_map}, we have:
  \begin{equation*}
    \forall x \in [-1,1]: \quad A_N(\set{x}) = \left\{
      \begin{array}{ll}
        \set{c_1}, & \text{if } x \in [-1,0],\\
        \set{c_2}, & \text{if } x \in (0,1],\\
      \end{array}
      \right.
    \end{equation*}
    while according to~\eqref{eq:bra_formulation}, we must have:
  \begin{equation*}
    \forall x \in [-1,1]: \quad \bra{A_N}(\set{x}) = \left\{
      \begin{array}{ll}
        \set{c_1}, & \text{if } x \in [-1,0),\\
        \set{c_1,c_2}, & \text{if } x = 0,\\
        \set{c_2}, & \text{if } x \in (0,1].\\
      \end{array}
      \right.
    \end{equation*}
\end{example}

As can be seen from Example~\ref{example:classifier_bra}, the tightest
robust approximation of a non-constant classifier may be
\emph{multi-valued} over certain parts of the input domain. 

\begin{theorem}
  \label{thm:classifier_robust_approx_multivalued}
  Assume that $N: [-M,M]^n \to C$ is a non-constant classifier, and let
  $\hat{N}: \kCPPO{[-M,M]^n} \to \kCPPO{C}$ be any robust approximation
  of the reachability map $A_N$. Then, $\hat{N}(\set{x})$ must be
  multivalued for some $x \in [-M,M]^n$.
\end{theorem}

\begin{proof}
  To obtain a contradiction, assume that $\hat{N}(x)$ is single-valued
  for all $x \in [-M,M]^n$. Therefore, the restriction of $\hat{N}$
  over $[-M,M]^n$---which is the set of maximal elements of
  $\kCPPO{[-M,M]^n}$---must be equal to $N$. By
  Corollary~\ref{cor:Scott_restrict_Euclid}, the classifier
  $N : [-M,M]^n \to C$ must be continuous, which, combined with
  Proposition~\ref{prop:classifier_cont_constant}, implies that $N$ is
  a constant classifier, which is a contradiction.
\end{proof}

In Theorem~\ref{thm:classifier_robust_approx_multivalued}, assume that
$C = \set{c_1, \ldots, c_k}$. In case $\hat{N}$ is the \emph{tightest}
robust approximation $\bra{A_N}$ of $A_N$, then $\hat{N}$ is required
to be multivalued only over the boundaries of
$N^{-1}(c_1), \ldots, N^{-1}(c_k)$. In typical applications, these
boundaries have Lebesgue measure zero, and the restriction of
$\bra{A_N}$ over $[-M,M]^n$ is single-valued and equal to $N$,
\emph{almost everywhere}, {\eg}, as in
Example~\ref{example:classifier_bra}. Computability of $\bra{A_N}$,
however, is an open problem in general, and even in cases where it is
known to be computable, little is known about complexity of its
computation. Thus, in practice, one aims for robust approximations
$\hat{N}$ that are `sufficiently' tight, and over $[-M,M]^n$, may be
multivalued on sets that are not Lebesgue measure zero.

\subsection{Regressors}
\label{seubsec:regressors}

Assume that a given feedforward regressor $N: [-M,M]^n \to \R$ has $H$
hidden layers. We let $n_0 \defeq n$, and for each $1 \leq i \leq H$,
let $n_i$ denote the number of neurons in the $i$-th layer. Then, for
each $1 \leq i \leq H$, we represent:

\begin{itemize}
\item the weights as a matrix $W_i$, with $\dim(W_i) = n_i \times
  n_{i-1}$;

  \item the biases as a vector $b_i$, with $\dim(b_i) = n_i \times
    1$;

  \item and the activation function (seen as a vector field) by
    $\sigma_i: \R^{n_i} \to \R^{n_i}$.
  
\end{itemize}
\noindent
From the last hidden layer to the single output neuron, the weights
are represented as the vector $W_{H+1}$, with
$\dim(W_{H+1}) = 1 \times n_H$, and there are no biases or activation
functions. For any input $x \in [-M,M]^n$, the output of the neural
network may be calculated layer-by-layer as follows:
\begin{equation}
  \label{eq:layer_by_layer}
  \left\{
  \begin{array}{ll}
  Z_0( x) = x,\\
  \hat{Z}_i(x) = W_i Z_{i-1}(x) + b_i, & (1 \leq i \leq H),\\
  Z_i(x) = \sigma_i \left(\hat{Z}_i(x) \right),  & (1 \leq i \leq H),\\
  N(x) = W_{H+1} Z_H (x).
  \end{array}
  \right.
\end{equation}
\noindent
We let ${\cal F}$ denote the class of feedforward networks with a
single output neuron, and Lipschitz continuous activation functions,
whose output may be described using~\eqref{eq:layer_by_layer}. Each
$N \in {\cal F}$ is continuous. Hence, it maps the compact set
$[-M,M]^n$ to a compact set, which we assume is included in
$[-M',M']$, for $M'$ large enough.

\begin{proposition}
  \label{prop:regressor_tightest_robust_extension}
  Assume that $N: [-M,M]^n \to [-M',M']$ is a regressor in ${\cal F}$,
  and $A_N$ is its reachability map. Then, the map $A_N$ has a
  tightest robust approximation which is a robust extension of~$N$,
  {\ie}, an extension according to
  Definition~\ref{def:ext_canonical_ext}~\ref{item:extension}.
\end{proposition}

\begin{proof}
  Since both $[-M,M]^n$ and $[-M',M']$ are compact, by
  Corollary~\ref{cor:tightest_robust_approximation}, the reachability
  map $A_N$ has a tightest robust approximation $\bra{A_N}: \kCPPO{[-M,M]^n} \to \kCPPO{[-M',M']}$. To prove that
  $\bra{A_N}$ is indeed an extension of $N$, we must show that over
  the maximal elements of $\kCPPO{[-M,M]^n}$, the map $\bra{A_N}$
  takes singleton values. This is a straightforward consequence of the
  fact that $N$ is continuous.
\end{proof}

This is in contrast with the case of classifiers, where the network
itself did not have to be even continuous, and we had to be content
with tightest robust approximations.

Of particular interest is the interval enclosure of a given regressor,
which may not be the tightest robust extension of the regressor, but
it is a robust extension nonetheless, which is commonly used in
interval analysis of neural networks:

\begin{theorem}
  \label{thm:regressor_interval_robust_extension}
  Let $N: [-M,M]^n \to [-M',M']$ be a regressor in ${\cal
    F}$, and assume that:
  \begin{itemize}
  \item
    $\hbClos{(\cdot)} : \kCPPO{[-M,M]^n} \to \intvaldom[{[-M,M]^n}]$
    is the hyperbox closure map of~\eqref{eq:hyper_rect_clos}.

  \item $\tilde{N}: \intvaldom[{[-M,M]^n}] \to \intvaldom[{[-M',M']}]$
    is any Scott-continuous extension of $N$, {\eg}, the canonical
    interval extension of
    Definition~\ref{def:ext_canonical_ext}~\ref{item:canonical_interval_extension}.

  \item $\iota : \intvaldom[{[-M',M']}] \to \kCPPO{[-M',M']}$ is the
    embedding of the subdomain $\intvaldom[{[-M',M']}]$ into
    $\kCPPO{[-M',M']}$, as in Corollary~\ref{cor:IR_subdomain_KR}.

  \end{itemize}
  \noindent
  Then, the map $\hat{N}: \kCPPO{[-M,M]^n} \to \kCPPO{[-M',M']}$
  defined by
  $\forall X \in \kCPPO{[-M,M]^n}: \hat{N}(X) \defeq
  \iota(\tilde{N}(\hbClos{X}))$ is a robust extension of $N$.
\end{theorem}

\begin{proof}
  We know that $\hbClos{(\cdot)}$, $\tilde{N}$, and $\iota$ are
  Scott-continuous. Therefore, the map $\hat{N}$---being a composition
  of the three---is Scott-continuous. Since both $[-M,M]^n$ and
  $[-M',M']$ are compact, by Theorem~\ref{thm:compact_metric}, the map
  $\hat{N}$ is robust. Hence, $\hat{N}$ is a robust approximation of
  $N$. To prove that $\hat{N}$ is a robust extension of $N$, we must
  show that over maximal input values it returns singletons. But this
  is again a straightforward consequence of continuity of $N$.
\end{proof}

  \section{Validated Local Robustness Analysis}
\label{sec:validated_lipschitz_constant}

The results of Section~\ref{sec:global_robustness_analysis} concern
globally robust approximation of neural networks. In the machine
learning literature, robustness analysis is mainly carried out
locally, that is, over specific locations in the input domain. A
classifier $N: \R^n \to \set{c_1, \ldots, c_k}$ may be regarded as
robust at $v \in \R^n$ if, for some neighborhood $B( \set{v}, r_0)$ of
radius $r_0 > 0$, we have:
$\forall x \in B(\set{v},r_0): N(x) = N(v)$. In the current article,
by attack-agnostic robustness measurement we mean obtaining a tight
lower bound for the largest such $r_0$. This, in turn, can be reduced
to obtaining a tight upper bound for the local Lipschitz constant in a
neighborhood of the
point~$v$~\parencite{Hein_Andriushchenko:Robustness:2017,Weng_et_al-CLEVER-ICLR:2018},
which is the focus of the rest of the current article.

From a theoretical angle, the more challenging networks for Lipschitz
analysis are those with non-differentiable activation functions, such
as $\ReLU$. For such functions, a generalized notion of gradient must
be used, {\eg},
Clarke-gradient~\parencite{Clarke:Opt_Non_Smooth_Analysis-Book:1990}. Indeed,
Clarke-gradient has been used recently for robustness analysis of
Lipschitz neural
networks~\parencite{Jordan_Dimakis:Exactly_NeurIPS:2020,Jordan_Dimakis:Provable_ICML:2021}. On
the other hand, the \emph{domain-theoretic $L$-derivative}
(Definition~\ref{def:L_derivative} \vpageref[below]{def:L_derivative})
introduced by~\textcite{Edalat:2008:Continuous_Derivative} also
coincides with Clarke-gradient, a fact which was proven for
finite-dimensional Banach spaces
by~\textcite{Edalat:2008:Continuous_Derivative}, and later generalized
to infinite-dimensional Banach spaces
by~\textcite{Hertling:Clarke_Edalat:2017}. Using this, we develop a
domain-theoretic framework for Lipschitz analysis of feedforward
networks, which provides a theoretical foundation for methods such
as~\parencite{Jordan_Dimakis:Exactly_NeurIPS:2020,Jordan_Dimakis:Provable_ICML:2021}.

Working with the $L$-derivative, however, is computationally
costly. It requires computation over $\cCPPO{\R^n\lift}$ and involves
a rather complicated chain rule
(Proposition~\ref{prop:Clarke_chain_rule}). Therefore, we consider a
hyperbox approximation of the $L$-derivative, which we call the
$\hat{L}$-derivative (Definition~\ref{def:Lhat_derivative}). This will
allow us to work with the relatively simpler domain
$\intvaldom[\R^n\lift]$ and also a much simpler chain rule
(Lemma~\ref{lem:chain_Rule_hat_L}). Finally, we will move on to
computation of the Lipschitz constant of a given regressor, according
to the following overall strategy.

Assume that $N \in {\cal F}$ is a feedforward regressor. By
using~\eqref{eq:layer_by_layer}, we obtain a unique closed-form
expression for $N$ in terms of the weights, biases, and the activation
functions. Let
$\tilde{L}(N): \intvaldom[{[-M,M]^n}] \to \intvaldom[\R^n\lift]$
denote the interval approximation of the $\hat{L}$-derivative of $N$
obtained by applying the chain rule of
Lemma~\ref{lem:chain_Rule_hat_L} over this closed-form expression. The
map $\tilde{L}(N)$ is Scott-continuous. Since we consider
$\norm{.}_{\infty}$ as the perturbation norm, we must compute the
$\norm{.}_1$-norm of the gradients (see
Remark~\ref{rem:func_analysis_dual_Banach}). First, we extend the
absolute value function to real intervals by defining
$\absn{I} \defeq \setbarNormal{\absn{x}}{x \in I}$ for any interval
$I$. Then, we extend the $\norm{.}_1$ norm to hyperboxes by defining
$\norm{\prod_{i=1}^n I_i}_1 \defeq \sum_{i=1}^n \absn{I_i}$. For
instance, we have
$\norm{[-1,1] \times [-1,2]}_1 = [0,1] + [0,2] = [0,3]$. The function
$\norm{.}_1 : \intvaldom[\R^n\lift] \to \intvaldom[\R^1\lift]$ is also
Scott-continuous. Therefore, the composition
$\norm{\tilde{L}(N)}_1: \intvaldom[{[-M,M]^n}] \to
\intvaldom[\R^1\lift]$ is a Scott-continuous function. All that
remains is to find the maximum value and the maximum set of
$\norm{\tilde{L}(N)}_1$ to approximate the Lipschitz constant and the
set of points where the maximum Lipschitz values are attained, within
any given neighborhood. This is one of the main contributions of the
current article and will be presented in detail in
Section~\ref{subsec:Maximization_algorithm}.

If, instead of $\intvaldom[\R^n\lift]$, we use the finer domain
$\cCPPO{\R^n\lift}$, then the suitable chain rule will be that of
Proposition~\ref{prop:Clarke_chain_rule}, which provides a more
accurate result, at the cost of calculating the closed convex
hull. This has also been observed
in~\parencite{Jordan_Dimakis:Provable_ICML:2021}. By using the
$\hat{L}$-derivative and the chain rule of
Lemma~\ref{lem:chain_Rule_hat_L}, we avoid computation of the convex
hull, and trade some accuracy for efficiency. As it turns out,
however, \emph{we will not lose any accuracy over general position
  $\ReLU$ networks (Theorem~\ref{thm:gen_pos_ReLU}), or over
  differentiable networks
  (Theorem~\ref{thm:differentiable_networks})}.

\begin{remark}
\label{rem:l_p_norms}
  In this article, we focus on the $\norm{.}_{\infty}$-norm for
  perturbations because we will be working with hyperboxes, and in the
  implementation, we will use interval arithmetic. This does not mean,
  however, that we cannot analyze robustness with respect to
  $\norm{.}_p$ for other values of $p$. For example, assume that we
  are given a closed neighborhood
  $V(x_0,r) \defeq \setbarNormal{x \in \R^n}{ \norm{x - x_0}_p \leq
    r}$, for some $p \in [1, \infty)$. We can first cover the
  neighborhood $V(x_0,r)$ with (a large number of sufficiently small)
  hyperboxes $B = \setbarNormal{b_i}{i \in I}$. We obtain an
  overapproximation of the $\norm{.}_{p'}$-norm of the Clarke-gradient
  over each box $b_i$, and then obtain an approximation of the
  Lipschitz constant over the neighborhood $V(x_0,r)$.
\end{remark}

\subsection{Lipschitz Continuity and Lipschitz Constant}

Assume that $(X, \delta_X)$ and $(Y, \delta_Y)$ are two metric spaces,
$f : X \to Y$, and $A \subseteq X$. A real number $L \geq 0$ is said
to be a \emph{Lipschitz bound} for $f$ over $A$ if:

\begin{equation}
  \label{eq:Lipschitz_bound}
  \forall a, b \in A: \delta_Y(f(a), f(b)) \leq L \delta_X( a, b).  
\end{equation}
\noindent
A function for which such a bound exists is said to be Lipschitz
continuous over $A$, and the smallest $L$
satisfying~\eqref{eq:Lipschitz_bound}---which must exist---is called
the \emph{Lipschitz constant} of $f$ over $A$.

If $f$ is Lipschitz continuous over $A$, then it is also continuous
over $A$, although the converse is not true. For instance, the
function $h : \ell^1_\infty \to \ell^1_\infty$, defined by
$h(x) \defeq \sqrt{\absn{x}}$, is continuous everywhere, but not
Lipschitz continuous over any neighborhood of $0$. Whenever $X$ is a
compact subset of $\ell^n_\infty$ and $f: X \to \R$ is continuously
differentiable, then $f$ is Lipschitz over $X$, but again, the
converse is not true. For example, the function $h: [-1,1] \to \R$,
defined by $h(x) \defeq \absn{x}$, is Lipschitz continuous, but not
differentiable at $0$. To summarize, we have the following inclusions
for functions over compact subsets of $\R^n$:
\begin{equation*}
  \text{Continuously Differentiable} \subset \text{Lipschitz
    Continuous} \subset \text{Continuous}.
\end{equation*}
\noindent
Even though the inclusions are strict, there is a close relationship
between Lipschitz continuity and differentiability. Rademacher's
theorem states that if $f$ is Lipschitz continuous over an open subset
$A$ of $\R^n$, then it is (Fr{\'e}chet) differentiable almost
everywhere (with respect to the Lebesgue measure) over
$A$~\parencite[Corollary~4.19]{Clarke_et_al:Nonsmooth_Control:Book:1998}. For
example, the function $\ReLU$---defined by
$\forall x \in \R: \ReLU(x) \defeq \max( 0, x)$---is (classically)
differentiable everywhere except at zero:
\begin{equation*}
  (\ReLU)'(x) =
  \left\{
    \begin{array}{ll}
      1, & \text{if } x > 0, \\
      0, & \text{if } x < 0, \\
      \mathrm{undefined}, & \text{if } x = 0.
    \end{array}
  \right.
\end{equation*}

\subsection{Clarke-Gradient}

Although $\ReLU$ is not differentiable at $0$, one-sided limits of the
derivative exist, {\ie}:
\begin{equation}
  \label{eq:relu_left_right_derive}
\lim_{x \to 0^-} (\ReLU)'(x) = 0 \quad \text{and} \quad \lim_{x \to 0^+} (\ReLU)'(x) = 1.   
\end{equation}
As shown in Figure~\ref{fig:relu_cones-step_fun} (Left), in geometric terms, for every
$\lambda \in [0,1]$, the line $y_\lambda(x) \defeq \lambda x$
satisfies:
\begin{equation}
  \label{eq:below_relu}
  y_\lambda (0) = \ReLU (0) \quad \wedge \quad \forall x \in \R: y_\lambda(x) \leq
  \ReLU (x).
\end{equation}
The interval $[0,1]$ is the largest set of $\lambda$ values which
satisfy~\eqref{eq:below_relu}, and may be considered as a gradient
\emph{set} for $\ReLU$ at zero. The concept of generalized (Clarke)
gradient formalizes this idea. Clarke-gradient may be defined for
functions over any Banach space $X$~\parencite[page
27]{Clarke:Opt_Non_Smooth_Analysis-Book:1990}. When $X$ is
finite-dimensional, however, the Clarke-gradient of $f: X \to \R$ has
a simpler characterization. Assume that $X$ is an open subset of
$\R^n$, and $f: X \to \R$ is Lipschitz continuous over~$X$. By
Rademacher's theorem, we know that the set
${\cal N}_f \defeq \setbarNormal{z \in X}{f \text{ is not
    differentiable at } z \, }$ has Lebesgue measure zero. The
Clarke-gradient $\clarke{f}(x)$ of $f$ at any $x \in X$ satisfies:
\begin{equation}
  \label{eq:clarke_fin_dim_charac}
  \clarke{f}(x) = \convHull \setbarTall{\lim_{i \to \infty} f'(x_i)}{ \lim_{i \to \infty} x_i = x \wedge \forall
    i\in \N: x_i \notin {\cal N}_f}.
\end{equation}
Equation~\eqref{eq:clarke_fin_dim_charac} should be interpreted as
follows: take any sequence $(x_i)_{i \in \N}$, which completely avoids
the set ${\cal N}_f$ of points where $f$ is not differentiable,
converges to $x$, and for which $\lim_{i \to \infty} f'(x_i)$
exists. The Clarke-gradient is the convex hull of these
limits~\parencite[page 63]{Clarke:Opt_Non_Smooth_Analysis-Book:1990}. It
should be clear from~\eqref{eq:relu_left_right_derive}
and~\eqref{eq:clarke_fin_dim_charac}
that $\clarke{\ReLU}(0) = [0,1]$.

\begin{proposition}[{\parencite[Proposition~2.1.2]{Clarke:Opt_Non_Smooth_Analysis-Book:1990}}]
  \label{prop:Clarke_general_properties}
  Assume that $p \in [1,\infty]$. Let $X$ be an open subset of
  $\ell^n_p$ and $f: X \to \R$ be Lipschitz continuous over $X$ with a
  Lipschitz bound $L$. Then, for every $x \in X$:

  \begin{enumerate}[label=(\roman*)] \item
  \label{item:clarke_convex_compact_subset_Rn} $\clarke{f}(x)$ is a
  non-empty, convex, and compact subset of $\R^n$.  \item $\forall \xi
  \in \clarke{f}(x): \norm{\xi}_{p'} \leq L$, in which $p'$ is the
  conjugate of $p$ as defined in~\eqref{eq:p_prime_conj}.
  \end{enumerate} \end{proposition}

The following chain rule for the Clarke-gradient is crucial for
proving the soundness of methods such
as~\parencite{Chaudhuri_et_al:Proving_Programs_Robust:2011,Chaudhuri_et_al:Continuity_Robustness:2012,Jordan_Dimakis:Exactly_NeurIPS:2020,Jordan_Dimakis:Provable_ICML:2021,Laurel_et_al:Dual_Number:2022}:

\begin{proposition}[\textbf{Chain Rule},
  {\parencite[Theorem~2.3.9]{Clarke:Opt_Non_Smooth_Analysis-Book:1990}}]
  \label{prop:Clarke_chain_rule}
  Assume that $X \subseteq \R^n$ is open,
  $h = (h_1, \ldots, h_m): X \to \R^m$ is Lipschitz continuous in a
  neighborhood of $x \in X$, and $g : \R^m \to \R$ is Lipschitz
  continuous in a neighborhood of $h(x) \in \R^m$. Then,
  $f \defeq g \circ h$ is Lipschitz in a neighborhood of $x$, and we
  have:
  \begin{equation}
    \label{eq:Clarke_chain_rule}
    \clarke{f}(x) \subseteq \cl{\convHull} \setbarTall{\sum_{i=1}^m
      \alpha_i \zeta_i}{ \zeta_i \in \clarke{h_i}(x) \wedge \alpha \in
    \clarke{g}(h(x))},
\end{equation}
in which, $\cl{\convHull}$ denotes the closed convex hull.
\end{proposition}

Unlike the chain rule for the classical derivative, the chain rule
of~\eqref{eq:Clarke_chain_rule} provides only a subset relation,
which, in general, is not an equality. A simple example of lack of
equality is shown in Figure~\ref{fig:dependency-zero_NN}, which
demonstrates an instance of the dependency problem (see
Section~\ref{subsubsec:dependency}).

\subsection{Domain-Theoretic \texorpdfstring{$\boldsymbol{L}$}{L}-Derivative}

We recall the concept of a domain-theoretic derivative for a function
$f: U \to \R$ defined on an open set $U \subseteq \R^n$. Assume that
$(X, \Omega)$ is a topological space, and $(D, \sqsubseteq)$ is a
\ac{PDCPO}, with bottom element~$\bot$. Then, for any open set
$O \in \Omega$, and any element $b \in D$, we define the single-step
function $b \chi_O: X \to D$ as follows:
\begin{equation}
  \label{eq:single_step_fun}
  b \chi_O (x) \defeq
  \left\{
    \begin{array}{ll}
      b, &  \text{if } x \in O,\\
      \bot, & \text{if } x \in X \setminus O.
    \end{array}
  \right.
\end{equation}  

\begin{figure}[t]
  \centering
  \scalebox{0.3}[0.3]{\includegraphics{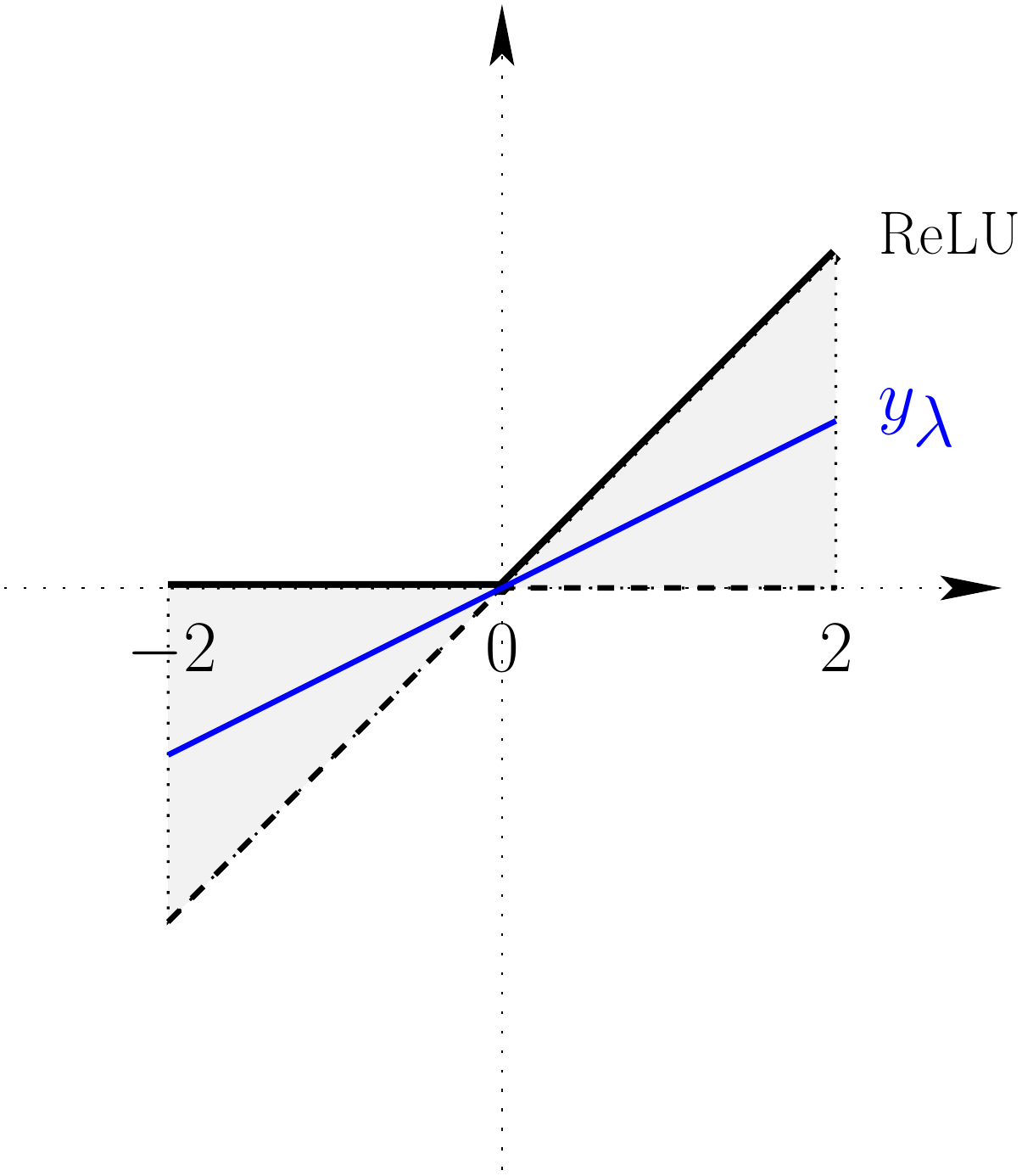}} \hspace{3em}
  \scalebox{0.3}[0.3]{\includegraphics{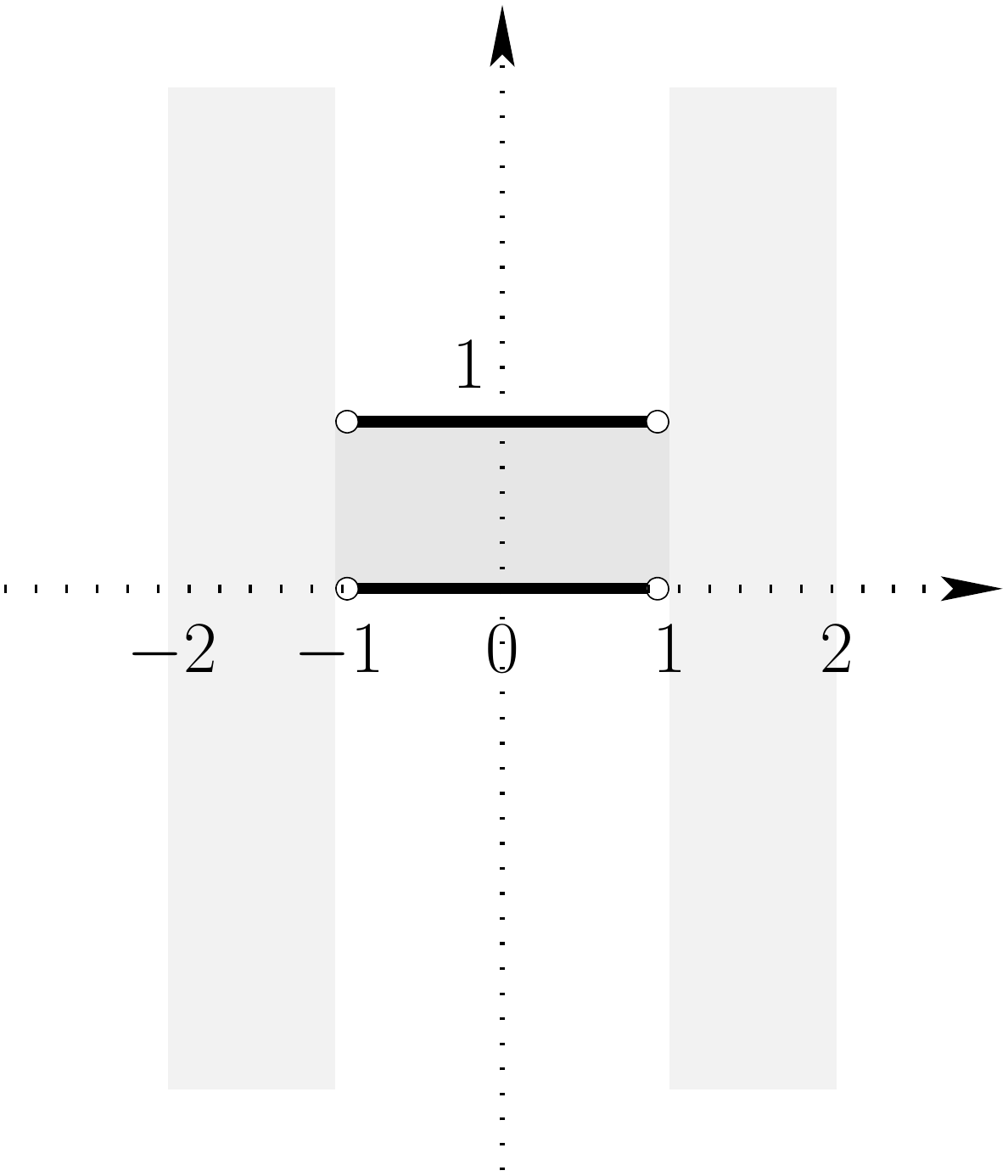}}
  
  \caption{Left: The function $\ReLU: (-2,2) \to \R$ (solid black) and
    the line $y_\lambda$ (in {\color{blue}{blue}})
    from~\eqref{eq:below_relu} for some $\lambda \in [0,1]$, which
    lies entirely in the shaded cone. Right: The step function
    $b \chi_O: (-2,2) \to \intvaldom[\R\lift]$ in which $b = [0,1]$
    and $O = (-1,1)$. We have $\ReLU \in \delta(O,b)$.}
  \label{fig:relu_cones-step_fun}
\end{figure}

\begin{definition}[$L$-derivative~{\parencite[Section~2]{EdalatLieutierPattinson:2013-MultiVar-Journal}}]
  \label{def:L_derivative}

  Assume that $O \subseteq U \subseteq \R^n$, both $O$ and $U$ are
  open, and $b \in \cCPPO{\R^n\lift}$:

  \begin{enumerate}[label=(\roman*)]

  \item The single-step tie $\delta(O, b)$ is the set of all functions
    $f : U \to \R$ that satisfy
    $\forall x, y \in O: b(x-y) \sqsubseteq f(x) - f(y)$. The set
    $b(x-y)$ is obtained by taking the inner product of every element
    of $b$ with the vector $x-y$, and $\sqsubseteq$ is the reverse
    inclusion order on $\cCPPO{\R^1\lift}$.

  \item \label{item:L_Derivative} The $L$-derivative of any function $f: U \to \R$ is defined
    as:    
    \begin{equation*}
     L(f) \defeq \bigsqcup \setbarTall{b \chi_O}{O \subseteq U, b \in
       \cCPPO{\R^n\lift}, f \in \delta(O, b)}. 
   \end{equation*}
   As such, the map $L$ has type $L: (U \to \R) \to (U \to \cCPPO{\R^n\lift})$.  
\end{enumerate}
  
\end{definition}

\begin{example}
  Assume that $U = (-2,2)$. Consider the function
  $\ReLU: (-2,2) \to \R$ and the step function
  $b \chi_O: (-2,2) \to \intvaldom[\R\lift]$ with $b = [0,1]$ and
  $O = (-1,1)$, as depicted in
  Figure~\ref{fig:relu_cones-step_fun}. We have
  $\ReLU \in \delta(O,b)$. In fact, we have $\ReLU \in \delta(O,b)$
  for any open set $O$ containing zero, {\eg}, the interval
  $(-2^n, 2^n)$, for any $n \in \N$. Thus, one may verify that:
  \begin{equation*}
    \forall x \in (-2,2): \quad  L(\ReLU)(x) =
    \left\{
      \begin{array}{ll}
        0, & \text{if } x < 0,\\
        {[0,1]}, & \text{if } x = 0,\\
        1, & \text{if } x > 0.
      \end{array}
    \right.    
  \end{equation*}
\end{example}

For any $f: U \to \R$, the $L$-derivative
$L(f): U \to \cCPPO{\R^n\lift}$ is Euclidean-Scott-continuous. When
$f$ is classically differentiable at $x \in U$, the $L$-derivative and
the classical derivative
coincide~\parencite{Edalat:2008:Continuous_Derivative}. Many of the
fundamental properties of the classical derivative can be generalized
to the domain-theoretic one, {\eg}, additivity and the chain
rule~\parencite{Edalat_Pattinson:2005:Inverse_Implicit}. Of particular
importance to the current paper is the following result:
\begin{theorem}[\parencite{Edalat:2008:Continuous_Derivative,Hertling:Clarke_Edalat:2017}]
\label{thm:L_deriv_Clarke_coincide}
  Over Lipschitz continuous functions $f: U \to \R$, the
  domain-theoretic $L$-derivative coincides with the Clarke-gradient,
  {\ie}, $\forall x \in U: L(f)(x) = \clarke{f}(x)$. 
\end{theorem}

In this paper, instead of working with the general convex sets in
$\cCPPO{\R^n\lift}$, we work with the simpler hyperboxes in
$\intvaldom[\R^n\lift]$:
\begin{definition}[$\hat{L}$-derivative]
  \label{def:Lhat_derivative}
  Let $U \subseteq \R^n$ be an open set. In the definition of
  $L$-derivative (Definition~\ref{def:L_derivative}), if
  $\cCPPO{\R^n\lift}$ is replaced with $\intvaldom[\R^n\lift]$, then
  we obtain the concept of $\hat{L}$-derivative for functions of type
  $U \to \R$. For vector valued functions
  $f = ( f_1, \ldots, f_m): U \to \R^m$, we define:
  \begin{equation*}
    \hat{L}(f) \defeq  \left( \hat{L}(f_1),  \ldots, \hat{L}(f_m)
    \right)^\intercal,
  \end{equation*}
  in which, $(\cdot)^\intercal$ denotes the transpose of a matrix. In
  other words, for each $1 \leq i \leq m$ and $x \in U$, let
  $\hat{L}(f_i)(x) \equiv ( \alpha_{i,1}, \ldots, \alpha_{i,n}) \in \intvaldom[\R^n\lift]$. Then, for
  each $x \in U$, $\hat{L}(f)(x)$ is the $m \times n$ interval matrix
  $[\alpha_{i,j}]_{1 \leq i \leq m, 1 \leq j \leq n}$.
\end{definition}

The $\hat{L}$-derivative is, in general, coarser than the
$L$-derivative, as the following example demonstrates:
\begin{example}
  \label{example:L_hat_overapprox_L_derive}
  Let $D_n$ denote the closed unit disc in $\R^n$, and define
  $f: \R^n \to \R$ by $f(x) = \norm{x}_2$, in which $\norm{\cdot}_2$
  is the Euclidean norm. Then, $L(f)(0) = D_n$, while
  $\hat{L}(f)(0) = [-1,1]^n$.
\end{example}

As such, we lose some precision by using hyperboxes. On the other
hand, on hyperboxes, the chain rule has a simpler form:
  \begin{lemma}[\textbf{Chain Rule on Hyperboxes},  {\parencite[Lemma~3.3]{Edalat_Pattinson:2005:Inverse_Implicit}}]
    \label{lem:chain_Rule_hat_L}
    For any two functions $g : \R^k \to \R^m$ and $f: \R^m \to \R^n$,
    we have:
    \begin{equation*}
      \hat{L}(f)(g(x))  \times  \hat{L}(g)(x)
      \sqsubseteq \hat{L}(f \circ g)(x).
    \end{equation*}
  \end{lemma}

 
\subsection{Maximization Algorithm}
\label{subsec:Maximization_algorithm}

For any closed compact interval $I$, we write
$I = [\lep[I], \uep[I]]$. Assume that
$f : \R^n \to \intvaldom[\R^1\lift]$ and $X \subseteq \R^n$ is a
compact set. We define the maximum value and maximum set of $f$ over
$X$ as follows:
\begin{equation}
  \label{eq:maxV_maxS}
  \left\{
    \arrayoptions{0.5ex}{1.3}
    \begin{array}{l}
      \maxV(f,X) \defeq [\lep[z],\uep[z]], \quad \text{where $\lep[z] \defeq \sup_{x \in X} \lep[f(x)]$ and
$\uep[z] \defeq \sup_{x \in X} \uep[f(x)]$ },  \\
      \maxS(f,X) \defeq \setbarTall{x \in X}{ f(x) \cap \maxV(f,X) \neq
      \emptyset}. 
    \end{array}
  \right.
\end{equation}
\noindent
For any $\hat{f}: \intvaldom[\R^n\lift] \to \intvaldom[\R^1\lift]$,
the maximum value and maximum set are defined by considering the
restriction of $\hat{f}$ over the maximal elements $\R^n$.

\begin{proposition}
  \label{prop:maxS_compact}
  Assume that $\hat{f}: \intvaldom[\R^n\lift] \to \intvaldom[\R^1\lift]$ is
  Scott-continuous, and $X$ is a compact subset of $\R^n$. Then
  $\maxS(\hat{f},X)$ is a compact subset of $\R^n$.
\end{proposition}

\begin{proof}
  Let $f : \R^n \to \intvaldom[\R^1\lift]$ be the restriction of
  $\hat{f}$ over the maximal elements $\R^n$. As $X$ is compact, it
  suffices to prove that $\maxS(f,X)$ is a closed subset of $X$. Let
  $x \in X \setminus \maxS(f,X)$ be an arbitrary point, and write
  $f(x) = [\lep[f(x)], \uep[f(x)]]$. Since $x \not \in \maxS(f,X)$,
  then $\uep[f(x)] < \lep[z]$, where
  $z \equiv [\lep[z], \uep[z]] = \maxV(f,X)$. As $\hat{f}$ is
  Scott-continuous, then, by
  Corollary~\ref{cor:Scott_restrict_Euclid}, the map $f$ must be
  Euclidean-Scott-continuous. In turn, by
  Proposition~\ref{prop:Euclidean_Scott_Cont}, the function $\uep[f]$
  must be upper semicontinuous. As $\uep[f(x)] < \lep[z]$, for some
  neighborhood $O_x$ of $x$, we must have:
  $\forall y \in O_x: \uep[f(y)] < \lep[z]$. This entails that
  $O_x \cap \maxS(f,X) = \emptyset$.
\end{proof}

Thus, in the sequel, we consider the functionals:
\begin{equation*}
  \left\{
    \arrayoptions{0.5ex}{1.3}
    \begin{array}{ll}
      \maxV: & (\intvaldom[\R^n\lift] \Rightarrow \intvaldom[\R^1\lift]) \to
      \kCPPO{\R^n\lift} \to \intvaldom[\R^1\lift],\\
      \maxS: & (\intvaldom[\R^n\lift] \Rightarrow \intvaldom[\R^1\lift]) \to
\kCPPO{\R^n\lift} \to \kCPPO{\R^n\lift},
    \end{array}
  \right.
\end{equation*}
\noindent
which, given any Scott-continuous
$f \in \intvaldom[\R^n\lift] \Rightarrow \intvaldom[\R^1\lift]$ and
compact set $X \in \kCPPO{\R^n\lift}$, return the maximum value and
the maximum set of $f$ over $X$, respectively. As an example, assume
that $f \in \intvaldom[\R^1\lift] \Rightarrow \intvaldom[\R^1\lift]$
is the canonical interval extension of
$\clarke{\ReLU}: \R \to \intvaldom[\R^1\lift]$. Then:
\begin{equation}
  \label{eq:maxV_maxS_not_monotone_in_X}
  \left\{
    \begin{array}{ll}
     \maxV(f,[-1,1]) = [1,1], & \maxS(f,[-1,1]) = [0,1], \\
\maxV(f,[-1,0]) = [0,1], & \maxS(f,[-1,0]) = [-1,0], \\
    \end{array}
  \right.
\end{equation}
\noindent
which shows that the functionals $\maxV$ and $\maxS$ are not even
monotone in their second argument. For fixed second arguments,
however, both are Scott-continuous in their first argument:

\begin{theorem}[\textbf{Scott-continuity}]
  \label{thm:maxV_maxS_Scott_Cont}
  For any $X \in \kCPPO{\R^n\lift}$, the functionals
  $\maxV(\cdot , X): (\intvaldom[\R^n\lift] \Rightarrow
  \intvaldom[\R^1\lift]) \to \intvaldom[\R^1\lift]$ and
  $\maxS( \cdot, X): (\intvaldom[\R^n\lift] \Rightarrow
  \intvaldom[\R^1\lift]) \to \kCPPO{\R^n\lift}$ are Scott-continuous.
\end{theorem}

\begin{proof}
  As the domains
  $\intvaldom[\R^n\lift] \Rightarrow \intvaldom[\R^1\lift]$,
  $\intvaldom[\R^1\lift]$, and $\kCPPO{\R^n\lift}$ are
  $\omega$-continuous, by Proposition~\ref{prop:Scott_w_cont}, it
  suffices to show that both functionals are monotone and preserve the
  suprema of $\omega$-chains.

To prove monotonicity, assume that
$f,g \in \intvaldom[\R^n\lift] \Rightarrow \intvaldom[\R^1\lift]$, and
$f \sqsubseteq g$. Thus,
$\forall x \in X: (\lep[f(x)] \leq \lep[g(x)]) \wedge (\uep[g(x)] \leq
\uep[f(x)])$. This implies that
$\sup_{x \in X} \lep[f(x)] \leq \sup_{x \in X} \lep[g(x)]$ and
$\sup_{x \in X} \uep[g(x)] \leq \sup_{x \in X} \uep[f(x)]$. Therefore:
  \begin{equation}
    \label{eq:maxV_f_g_monotone}
    \maxV(f,X) \sqsubseteq \maxV(g,X).
  \end{equation}
  Take an arbitrary $x \in \maxS(g,X)$. By definition, we must have
  $g(x) \cap \maxV(g,X) \neq \emptyset$, which, combined with the
  assumption $f \sqsubseteq g$, implies
  $f(x) \cap \maxV(g,X) \neq \emptyset$. This, combined
  with~\eqref{eq:maxV_f_g_monotone}, entails that
  $f(x) \cap \maxV(f,X) \neq \emptyset$. Hence, $x \in \maxS(f,X)$,
  and we have $\maxS(f,X) \sqsubseteq \maxS(g,X)$. The proof of
  monotonicity is complete.

  Next, let $(f_n)_{n \in \N}$ be an $\omega$-chain in
  $\intvaldom[\R^n\lift] \Rightarrow \intvaldom[\R^1\lift]$, with
  $f = \bigsqcup_{n \in \N} f_n$. As both $\maxV$ and $\maxS$ are
  monotone, we have
  $\bigsqcup_{n \in \N} \maxV(f_n,X) \sqsubseteq \maxV(f,X)$ and
  $\bigsqcup_{n \in \N} \maxS(f_n,X) \sqsubseteq \maxS(f,X)$. Thus, it
  only remains to prove the $\sqsupseteq$ direction in both cases. In
  what follows, we let
  $z \equiv [\lep[z], \uep[z]] \defeq \maxV(f,X)$, and for each
  $n \in \N$, $z_n \equiv [\lep[z_n], \uep[z_n]] \defeq \maxV(f_n,X)$.

\begin{description}
\item[$\maxV$:] First, consider an arbitrary $t < \lep[z]$ and let
  $\epsilon \defeq (\lep[z]-t)/4$. By definition,
  $\exists x_0 \in X: \lep[f(x_0)] > \lep[z] - \epsilon$. Since
  $f = \bigsqcup_{n \in \N} f_n$, we must have
  $\exists N_0 \in \N: \forall n \geq N_0: \lep[f_n(x_0)] >
  \lep[f(x_0)] - \epsilon > \lep[z]- 2\epsilon > t$.
  Therefore:
  \begin{equation}
    \label{eq:t_less_lep_z}
    \forall t \in \R: \quad t < \lep[z] \implies \exists n \in \N: t <
    \lep[z_n].
  \end{equation}
  \noindent
  Next, we consider an arbitrary $t > \uep[z]$ and let
  $\epsilon \defeq (t - \uep[z])/4$. Since
  $f = \bigsqcup_{n \in \N} f_n$, we must have:
  \begin{equation*}
    \forall x \in X: \exists N_x \in \N: \forall n \geq N_x: \quad
    \uep[f_n(x)] < \uep[f(x)] + \epsilon.
  \end{equation*}
  For each fixed $n \in \N$, the map $\uep[f_n]$ is upper
  semicontinuous. Hence, for any $x \in X$, there exists a
  neighborhood $O_x$ of $x$ for which we have
  \begin{equation*}
    \forall y \in O_x: \quad \uep[f_{N_x}(y)] < \uep[f_{N_x}(x)] +
    \epsilon < \uep[f(x)] + 2\epsilon \leq \uep[z] + 2\epsilon.
  \end{equation*}
  As the sequence $(f_n)_{n \in \N}$ is assumed to be an
  $\omega$-chain, the upper bounds $(\uep[f_n])_{n \in \N}$ are
  non-increasing, and we obtain:
  \begin{equation*}
    \forall x \in X: \forall y \in O_x: \forall n \geq N_x: \quad
    \uep[f_n(y)] \leq \uep[z] + 2\epsilon.
  \end{equation*}
  As $X$ is compact and $X \subseteq \bigcup_{x \in X} O_x$, then for
  finitely many neighborhoods $\set{O_{x_1}, \ldots, O_{x_m}}$ we must
  have $X \subseteq \bigcup_{i=1}^m O_{x_i}$. Thus, if we let
  $N_0 \defeq \max\set{N_{x_1}, \ldots, N_{x_m}}$, we must have:
  \begin{equation*}
    \forall n \geq N_0: \forall y \in X: \quad \uep[f_n(y)] \leq
    \uep[z] + 2\epsilon,
  \end{equation*}
  which entails that:
  \begin{equation}
    \label{eq:t_ht_uep_z}
    \forall t \in \R: \quad \uep[z] < t \implies \exists N_0 \in \N:
    \forall n \geq N_0:
    \uep[z_n] \leq
    \uep[z] + 2\epsilon < t.
  \end{equation}
  From~\eqref{eq:t_less_lep_z} and~\eqref{eq:t_ht_uep_z}, we deduce
  that $\bigsqcup_{n \in \N} \maxV(f_n,X) \sqsupseteq \maxV(f,X)$.

\item[$\maxS$:] Assume that $x \in X \setminus \maxS(f,X)$, which
  implies that $\uep[f(x)] < \lep[z]$. Take
  $\epsilon = \left(\lep[z] - \uep[f(x)]\right)/4$. Since
  $\lep[z] = \sup_{y \in X} \lep[f(y)]$, there exists $x_1 \in X$ such
  that $\lep[f(x_1)] > \lep[z] - \epsilon$. As
  $f = \bigsqcup_{n \in \N} f_n$, there exists $N_1 \in \N$, such
  that:
  \begin{equation}
    \label{eq:lep_fn_estimate}
    \forall n \geq N_1: \lep[f_n(x_1)] > \lep[f(x_1)] - \epsilon >
    \lep[z] - 2\epsilon.  
  \end{equation}
  Furthermore, there exists $N_2 \in \N$, such that:
  \begin{equation}
    \label{eq:uep_fn_estimate}
    \forall n \geq
    N_2: \uep[f_n(x)] < \uep[f(x)] + \epsilon = \lep[z] - 3\epsilon.
  \end{equation}
  Taking $N_3 = \max\set{N_1,N_2}$, from~\eqref{eq:lep_fn_estimate}
  and \eqref{eq:uep_fn_estimate}, we obtain
  $\forall n \geq N_3: \uep[f_n(x)] < \lep[f_n(x_1)] \leq \lep[z_n]$,
  which implies that $x \not \in \maxS(f_n,X)$. Therefore,
  $\bigsqcup_{n \in \N} \maxS(f_n,X) \sqsupseteq \maxS(f,X)$.
\end{description}
\noindent
The proof is complete.
\end{proof}

Algorithm~\ref{alg:interval_max} presents an implementation of $\maxV$
and $\maxS$, which we use in our experiments. As pointed out before,
the set $B_{\kCPPO{\R^n\lift}}$, consisting of finite unions of
rational hyperboxes, forms a countable basis for the
$\omega$-continuous domain $\kCPPO{\R^n\lift}$. We have used this fact
in Algorithm~\ref{alg:interval_max}, which outputs enclosures of the
maximum set using elements of $B_{\kCPPO{\R^n\lift}}$.

Apart from the function $f$ and the neighborhood $X$,
Algorithm~\ref{alg:interval_max} requires three other input
parameters: $m_{\iota}$, $m_b$, and $\delta$. These are crucial for making
sure that the algorithm stops below a reasonable threshold. In
line~\ref{alg:line:retain_boxes}, we discard all the boxes in $B_i$
which cannot possibly contain any maximum points. How many boxes are
discarded depends on a number of factors, most importantly, whether
the function $f$ has a flat section around the maximum set or not. The
condition $2^n\absn{\hat{B}_i} \geq m_b$ in
line~\ref{alg:line:stop_cond} is necessary, because in cases where
flat sections exist, the number of elements of $\hat{B}_i$ grows
exponentially with each iteration. In the \texttt{if} statement of
line~\ref{alg:line:if:bisect}, we make sure that bisecting $\hat{B}_i$
will not create too many boxes before executing the bisection command
in the first place.

\begin{algorithm}[t]
  \caption{Enclosing $\maxV(f,X)$ and $\maxS(f,X)$.}

\begin{algorithmic}[1]
  \REQUIRE {\ }

  $X \subseteq \R^n$ : A finite union of hyperboxes;
  
  $f \in \intvaldom[X] \Rightarrow \intvaldom[\R^1\lift]$:  A Scott-continuous interval function;

  $m_{\iota} \in \N$: maximum iterations;

  $m_b \in \N$: maximum number of boxes;
  
  $\delta \geq 0$: a dyadic number indicating target width.

  \ENSURE Enclosures of $\maxV(f,X)$ and $\maxS(f,X)$

 \STATE $i \gets 0$ \quad // number of iterations  

  \STATE $B_1 \gets X$ \quad // initialize the set of hyperboxes

  \REPEAT
  
 \STATE $i \gets i+1$
  
 \STATE $L_i\defeq \setbarNormal{ f(b)}{ b \in B_i}$

  \STATE $\uep[z_i] \gets \max \setbarNormal{\uep[I]}{ I \in L_i}$  // the maximum upper end point

 \STATE $\lep[z_i] \gets \max \setbarNormal{\lep[I]}{ I \in L_i}$ 
 // the maximum lower end point

 \STATE  $I_i \defeq [{\lep[z_i], \uep[z_i]}]$ \quad // maximum value

 \STATE \label{alg:line:retain_boxes}
 $\hat{B}_i \gets \setbarNormal{b \in B_i}{\lep[z_i] \in f(b) }$ \quad // retain only those boxes
 that could contain maximum points

 \IF {$(2^n\absn{\hat{B}_i} < m_b)$} {\label{alg:line:if:bisect}}
\STATE \label{alg:line:bisect} $B_{i+1} \gets \mathtt{bisect}(\hat{B}_i)$
\quad // bisect all boxes in $\hat{B}_i$, along all dimensions

\ENDIF
 
 \UNTIL{$(i \geq m_{\iota})$ or $(2^n\absn{\hat{B}_i} \geq m_b)$ or
   $ (\uep[z_i] - \lep[z_i] \leq \delta )$} \label{alg:line:stop_cond}

\RETURN $(I_i, \hat{B}_i)$
  
\end{algorithmic}
\label{alg:interval_max}
\end{algorithm}

Algorithm~\ref{alg:interval_max} closely follows the definition of
$\maxV$ and $\maxS$ given in~\eqref{eq:maxV_maxS}. Hence, the proof of
soundness is straightforward:

\begin{lemma}[\textbf{Soundness}]
  \label{lemma:soundness_maximization}
  For any given $f$, $X$, $m_{\iota}$, $m_b$, and $\delta$,
  Algorithm~\ref{alg:interval_max} halts in finite time and returns
  enclosures of $\maxV(f,X)$ and $\maxS(f,X)$.
\end{lemma}
\begin{proof}
  We prove the lemma by induction over $i$, that is, the number of
  iterations of the main \texttt{repeat-until} loop. When $i=1$, we
  have $B_i = X$, which trivially contains $\maxS(f, X)$.

For the inductive step, assume that at iteration $i$, before the start
of the loop, the finite set of hyperboxes $B_i$ satisfies
$\maxS(f,X) \subseteq \bigcup B_i$. Assume that
$\tilde{f} \in X \Rightarrow \intvaldom[\R^1\lift]$ is the restriction of
the interval function $f$ over the maximal elements $X$, which entails
that $f$ is an interval approximation of $\tilde{f}$. As a result,
$I_i$ is an overapproximation of $\maxV(\tilde{f},X)$, which entails
that the set $\hat{B}_i$---obtained in
line~\ref{alg:line:retain_boxes}---is an overapproximation of
$\maxS(\tilde{f},X)$.

The fact that the algorithm always halts in finite time follows from
the stopping condition $i \geq m_{\iota}$ of line~\ref{alg:line:stop_cond}.
\end{proof}

The algorithm combines the definition of $\maxV$ and $\maxS$ given
in~\eqref{eq:maxV_maxS} with a bisection operation that increases the
accuracy at each iteration. We prove that the result can be obtained
to any degree of accuracy. To that end, we first define the following
sequence:
\begin{equation}
  \label{eq:X_i_Seq}
  \left\{
    \begin{array}{ll}
      X_0 \defeq X, &\\
      X_{i+1} \defeq \mathtt{bisect}(X_i), & (\forall i \in \N).
    \end{array}
  \right.
\end{equation}
Note that by bisecting a box $b \subseteq \R^n$ we mean bisecting
along all the $n$ coordinates, which generates $2^n$ boxes. It is
straightforward to verify that
$\forall i \in \N: \hat{B}_i \subseteq B_i \subseteq X_i$.

\begin{theorem}[\textbf{Completeness}]
  \label{thm:completeness}
  Assume that $\Psi(\hat{f}, X, m_{\iota}, m_b, \delta)$ is the function
  implemented by Algorithm~\ref{alg:interval_max}. For any given $\hat{f}$
  and $X$:
  \begin{equation*}
    \lim_{m_{\iota} \to \infty, m_b \to \infty, \delta \to 0} \Psi(\hat{f}, X,
    m_{\iota}, m_b, \delta) = \left( \maxV(\hat{f},X), \maxS(\hat{f},X) \right).
  \end{equation*}
\end{theorem}

\begin{proof}
  We write the restriction of
  $\hat{f} \in \intvaldom[X] \Rightarrow \intvaldom[\R^1\lift]$ over the
  maximal elements as
  $f \equiv [\lep[f], \uep[f]]: X \to \intvaldom[\R^1\lift]$. As
  $\hat{f}$ is Scott-continuous, by
  Corollary~\ref{cor:Scott_restrict_Euclid} and
  Proposition~\ref{prop:Euclidean_Scott_Cont}, $\lep[f]$ and $\uep[f]$
  are lower and upper semicontinuous, respectively.

  \begin{description}
  \item[$\maxV$:]   Let $z \equiv [\lep[z], \uep[z]] \defeq \maxV(\hat{f},X)$. As
  $\uep[f]$ is upper semicontinuous, then it must attain its maximum
  over the compact set $X$ at some point $x_0 \in X$. Thus,
  $\forall x \in X: \uep[f](x) \leq \uep[f](x_0)$. Suppose that
  $\epsilon > 0$ is given. By upper semicontinuity of $\uep[f]$, for
  each $x \in X$, there exists a neighborhood $O_x$ such that
  $\forall y \in O_x: \uep[f](y) \leq \uep[f](x) + \epsilon \leq
  \uep[f](x_0) + \epsilon$. By referring to the sequence
  $(X_i)_{i \in \N}$ defined in~\eqref{eq:X_i_Seq}, and using Scott
  continuity of $\hat{f}$, we deduce that, for each $x \in X$, there
  exists a sufficiently large index $N_x \in \N$ that satisfies the
  following property:

\begin{itemize}
\item There exist $k_x \in \N$ and finitely many boxes
  $\setbarNormal{B_{N_x, j}}{ 1 \leq j \leq k_x}$ such that:
  \begin{equation}
    \label{eq:x_cup_B_N_x_j}
    x \in \interiorOf{\left(\bigcup_{1 \leq j \leq k_x} B_{N_x,
          j}\right)} \text{ and } \left(\bigcup_{1 \leq j \leq k_x} B_{N_x, j}\right) \subseteq O_x,     
  \end{equation}
  \noindent
  in which $\interiorOf{(\cdot)}$ denotes the interior operator, and:
  \begin{equation}
    \label{eq:B_N_x_j}
    \forall j \in \set{1, \ldots, k_x}: \quad    \uep[\hat{f}(B_{N_x, j})] \leq \uep[f](x_0) +
    \epsilon.  
  \end{equation}
\end{itemize}
The reason for considering $k_x$ boxes
$\setbarNormal{B_{N_x, j}}{ 1 \leq j \leq k_x}$---rather than just one
box---is that, if at some iteration $i_0 \in \N$, the point $x$ lies
on the boundary of a box in $X_{i_0}$, then it will remain on the
boundary of boxes in $X_i$ for all $i \geq i_0$. In other words, it
will not be in the interior of any boxes in subsequent bisections. Of
course, as each point $x \in X$ may be on the boundaries of at most
$2^n$ adjacent boxes in $X_{N_x}$, then we can assume that
$k_x \leq 2^n$.

For each $x \in X$, we define the neighborhood
$U_x \defeq \interiorOf{(\bigcup_{1 \leq j \leq k_x} B_{N_x, j})}$. As
$X$ is compact, then it may be covered by finitely many of such
neighborhoods (say) $\setbarNormal{U_{x_i}}{1 \leq i \leq m}$. By
taking $N = \max_{1 \leq i \leq m} N_{x_i}$, and
considering~\eqref{eq:B_N_x_j}, we deduce that, at the latest, at
iteration $N$, we must have $\uep[z_i] \leq \uep[z] + \epsilon$.

One the other hand, as $\lep[z] \defeq \sup_{x \in X} \lep[f(x)]$, for
some $x'_0 \in X$ we must have $\lep[f](x'_0) > \lep[z] -
\epsilon/2$. Furthermore, as $\lep[f]$ is lower semicontinuous, then
for some neighborhood $X'_0$ of $x'_0$, we must have
$\forall y \in X'_0: \lep[f](y) > \lep[f](x'_0) - \epsilon/2 > \lep[z]
- \epsilon$. Similar to the previous argument, we can show that, after
a sufficient number of bisections, we must have
$\lep[z_i] \geq \lep[z] - \epsilon$.

\item[$\maxS$:] Assume that $x \in X \setminus \maxS(\hat{f},X)$. By
  Proposition~\ref{prop:maxS_compact}, the set $\maxS(\hat{f},X)$ is
  compact. Hence, the set $X \setminus \maxS(\hat{f},X)$ is relatively
  open in $X$ and there exists a neighborhood $O_x$ of $x$ which
  satisfies $\forall y \in O_x: \uep[f](y) < \lep[z]$. Once again, we
  can show that, after a sufficiently large number of bisections (say)
  $N_x$, for a finite number of boxes, the
  relation~\eqref{eq:x_cup_B_N_x_j} holds. This entails that all these
  boxes must be discarded in line~\ref{alg:line:retain_boxes} of the
  algorithm after at most $N_x$ iterations.
\end{description}
\end{proof}

\subsection{Computability}
\label{subsec:computability}

For computable analysis of the operators discussed in this work, we
need an effective structure on the underlying domains. To that end, we
use the concept of effectively given
domains~\parencite{Smyth:effectively_given_domains:1977}. Specifically, we follow the approach
taken
in~\parencite[Section~3]{Edalat-Sunderhauf-dt-computability-on-reals99}.

Assume that $(D, \sqsubseteq)$ is an $\omega$-continuous domain, with
a countable basis $B$ that is enumerated as follows:
    \begin{equation}
      \label{eq:basis_enum}
      B = \set{ b_0 = \bot, b_1, \ldots, b_n, \ldots}.
    \end{equation}
    We say that the domain $D$ is \emph{effectively given} with
    respect to the enumeration~\eqref{eq:basis_enum}, if the set
    $\setbarNormal{(i,j) \in \N \times \N}{ b_i \ll b_j}$ is
    recursively enumerable, in which $\ll$ is the way-below relation
    on $D$.

    In the following proposition, computability is to be understood
    according to \acl{TTE}~\parencite{Weihrauch2000:book}.

    \begin{proposition}[Computable elements and functions]
    \label{prop:comp_elem_fun}
    Let $D$ and $E$ be two effectively given domains, with enumerated
    bases $B_1 = \set{d_0, d_1, \ldots, d_n, \ldots}$ and
    $B_2 = \set{e_0, e_1, \ldots, e_n, \ldots}$, respectively. An
    element $x \in D$ is computable iff the set
    $\setbarNormal{i \in \N}{d_i \ll x}$ is recursively enumerable. A
    map $f: D \to E$ is computable iff
    $\setbarNormal{(i,j) \in \N \times \N}{ e_i \ll f(d_j)}$ is
    recursively enumerable.
  \end{proposition}

  \begin{proof}
    See ~\parencite[Proposition~3 and Theorem~9]{Edalat-Sunderhauf-dt-computability-on-reals99}.
  \end{proof}

Arithmetic operators have computable domain-theoretic extensions, so
do the common activation functions ({\eg}, $\ReLU$, $\tanh$, or
sigmoidal) and their
Clarke-gradients~\parencite{Edalat95:DT-fractals,Escardo96-tcs,Farjudian:Shrad:2007,EdalatLieutierPattinson:2013-MultiVar-Journal}. Our
aim here is to study the computability of $\maxV$ and
$\maxS$. From~\eqref{eq:maxV_maxS_not_monotone_in_X}, we know that
$\maxV$ and $\maxS$ are not even monotone in their second argument,
while, by Theorem~\ref{thm:maxV_maxS_Scott_Cont}, if we fix
$X \in \kCPPO{\R^n\lift}$, then the functionals
$\maxV(\cdot , X): (\intvaldom[\R^n\lift] \Rightarrow
\intvaldom[\R^1\lift]) \to \intvaldom[\R^1\lift]$ and
$\maxS( \cdot, X): (\intvaldom[\R^n\lift] \Rightarrow
\intvaldom[\R^1\lift]) \to \kCPPO{\R^n\lift}$ are Scott-continuous. By
referring to Proposition~\ref{prop:comp_elem_fun}, to study
computability of these two operators, we need to enumerate specific
countable bases for the domains
$\intvaldom[\R^n\lift] \Rightarrow \intvaldom[\R^1\lift]$,
$\intvaldom[\R^1\lift]$, and $\kCPPO{\R^n\lift}$. We already know that
$B_{\intvaldom[\R^1\lift]}$ and $B_{\kCPPO{\R^n\lift}}$ form countable
bases for $\intvaldom[\R^1\lift]$ and $\kCPPO{\R^n\lift}$,
respectively, and it is straightforward to effectively enumerate them.

  For the function space
  $\intvaldom[\R^n\lift] \Rightarrow \intvaldom[\R^1\lift]$, we use
  the common step function construction. Recall the single-step
  functions as defined in~\eqref{eq:single_step_fun}. Assume that $D$
  and $E$ are two domains. For any $d \in D$ and $e \in E$, we must
  have:
  \begin{equation*}
    e \chi_{\wayaboves{d}} (x) \defeq
    \left\{
    \begin{array}{ll}
      e, &  \text{if } d \ll x,\\
      \bot, & \text{if } d \not \ll x.
    \end{array}
  \right.
\end{equation*}  
Note that, by Proposition~\ref{prop:Scott_wayaboves}, the set
$\wayaboves{d}$ is Scott-open. By a step function we mean the supremum
of a finite set of single-step functions. If the domains $D$ and $E$
are bounded-complete, with countable bases
$B_D = \set{d_0, \ldots, d_n, \ldots}$ and
$B_E = \set{e_0, \ldots, e_n, \ldots}$, then the countable collection:
\begin{equation}
  \label{eq:B_D_to_E}
 B_{D \Rightarrow E} \defeq \setbarTall{\lub_{i \in I} e_i
   \chi_{\wayaboves{d_i}}}{I \text{ is finite}, \setbarNormal{e_i
   \chi_{\wayaboves{d_i}}}{i \in I} \text{ is bounded}}
\end{equation}
forms a basis for the function space
$D \Rightarrow E$~\parencite[Section~4]{AbramskyJung94-DT}.

Thus, we consider the bases $B_{\intvaldom[\R^n\lift]}$ and
$B_{\intvaldom[\R^1\lift]}$, and construct the countable basis
$B_{\intvaldom[\R^n\lift] \Rightarrow \intvaldom[\R^1\lift]}$ for the
function space
$\intvaldom[\R^n\lift] \Rightarrow \intvaldom[\R^1\lift]$, according
to~\eqref{eq:B_D_to_E}. The key point is that, each step function in
$B_{\intvaldom[\R^n\lift] \Rightarrow \intvaldom[\R^1\lift]}$ maps any
given hyperbox with rational coordinates to a rational interval. We
formulate this as:

\begin{proposition}
  \label{prop:rational_step_fun}
  Each step function in
  $B_{\intvaldom[\R^n\lift] \Rightarrow \intvaldom[\R^1\lift]}$ maps
  any given element of $B_{\intvaldom[\R^n\lift]}$ to an element of~$B_{\intvaldom[\R^1\lift]}$. 
\end{proposition}

As such, we call the elements of
$B_{\intvaldom[\R^n\lift] \Rightarrow \intvaldom[\R^1\lift]}$ rational
step functions.

\begin{theorem}[\textbf{Computability}]
  \label{thm:computability}
  Assume that $X \in B_{\kCPPO{\R^n\lift}}$ is fixed, that is, $X$ is
  a finite union of hyperboxes with rational coordinates. Then,
  the functionals
  $\maxV(\cdot , X): (\intvaldom[\R^n\lift] \Rightarrow
  \intvaldom[\R^1\lift]) \to \intvaldom[\R^1\lift]$ and
  $\maxS( \cdot, X): (\intvaldom[\R^n\lift] \Rightarrow
  \intvaldom[\R^1\lift]) \to \kCPPO{\R^n\lift}$ are computable.  
\end{theorem}

\begin{proof}
  We prove computability according to
  Proposition~\ref{prop:comp_elem_fun}. Let
  $f \in B_{\intvaldom[\R^n\lift] \Rightarrow \intvaldom[\R^1\lift]}$
  be a rational step function.

  \begin{description}
  \item[$\maxV$:] Assume that
    $e \equiv [\lep[e], \uep[e]]\in B_{\intvaldom[\R^1\lift]}$, that
    is, $e$ is a rational interval. To decide
    $e \ll \maxV(f, X) = [\lep[z], \uep[z]]$, we run
    Algorithm~\ref{alg:interval_max}. As $\wayaboves{e}$ is
    Scott-open, then by completeness (Theorem~\ref{thm:completeness}),
    after sufficiently large number of iterations (say) $N_f$, we must
    have $\forall i \geq N_f: e \ll [{\lep[z_i],
      \uep[z_i]}]$. By~\eqref{eq:way_below_compact_sets}, the relation
    $e \ll [{\lep[z_i], \uep[z_i]}]$ is equivalent to
    $\lep[e] < \lep[z_i] \leq \uep[z_i] < \uep[e]$. By
    Proposition~\ref{prop:rational_step_fun}, both $\lep[z_i]$ and
    $\uep[z_i]$ are rational numbers. Hence, deciding
    $e \ll [{\lep[z_i], \uep[z_i]}]$ reduces to comparing rational
    numbers, which is decidable. Thus, overall, the problem
    $e \ll \maxV(f, X)$ is semi-decidable.

  \item[$\maxS$:] The proof for $\maxS$ is similar to that for
    $\maxV$. Assume that $C \in B_{\kCPPO{\R^n\lift}}$. To decide
    $C \ll \maxS(f, X) = K$, we run
    Algorithm~\ref{alg:interval_max}. As $\wayaboves{C}$ is
    Scott-open, then by completeness (Theorem~\ref{thm:completeness}),
    after sufficiently large number of iterations (say) $N_f$, we must
    have $\forall i \geq N_f: C \ll
    \hat{B}_i$. By~\eqref{eq:way_below_compact_sets}, the relation
    $C \ll \hat{B}_i$ is equivalent to
    $\hat{B}_i \subseteq \interiorOf{C}$. As we have assumed that
    $C \in B_{\kCPPO{\R^n\lift}}$ and $X \in B_{\kCPPO{\R^n\lift}}$,
    then all the hyperboxes in $C$ and $\hat{B}_i$ have rational
    coordinates. Therefore, deciding $C \ll \hat{B}_i$ reduces to
    comparing rational numbers, which is decidable. Thus, overall, the
    problem $C \ll \maxS(f, X)$ is semi-decidable.
  \end{description}
\end{proof}


  \subsection{Interval Arithmetic}
  \label{subsec:interval_arithmetic}

We will use arbitrary-precision interval arithmetic for approximation
of the $\hat{L}$-derivative. Interval
arithmetic~\parencite{Moore:2009:IIA} is performed over sets of numbers,
in contrast with the classical arithmetic which is performed over
individual numbers. This allows incorporating all the possible errors
({\eg}, round-off and truncations errors) into the result. As an
example, note that:
  \begin{equation}
    \label{eq:interval_sum}
    (x \in [\lep[I],\uep[I]]) \wedge (y \in [\lep[J],\uep[J]]) \implies  x+y \in [\lep[I]+\lep[J], \uep[I]+\uep[J]].
  \end{equation}
In practice, inaccuracies enter into interval computations for a
variety of reasons. The following are the most relevant to the current article:

\subsubsection{Finite Representations}

In arbitrary-precision interval arithmetic, every real number $x$ is
represented as the limit of a sequence $(I_i)_{i \in \N}$ of intervals
with dyadic end-points, {\ie}, $[x,x] = \bigcap_{i \in \N} I_i$. Assume
that we want to compute $\exp([1,1]) = [\me, \me]$, in which $\me$ is
the Euler number. As $\me$ is an irrational number, it does not have a
finite binary expansion. For instance, in
MPFI~\parencite{Revol_Rouillier:MPFI:05}---the arbitrary-precision
interval library that we use---when the precision is set at $10$, the
result of $\exp([1,1])$ is given as $[2.718281828,2.718281829]$.

Furthermore, as the precision of the end-points of each interval must
be set first, inaccuracy is incurred even on simpler computations,
such as~\eqref{eq:interval_sum}. For instance, if the bit size is set
at $n$, then the result of $[0,0]+[-2^{-2n}, 2^{-2n}]$ must be rounded
outwards to $[-2^{-n}, 2^{-n}]$.

This source of inaccuracy is inevitable and, to a large extent,
harmless. The reason is that higher accuracy may be obtained by simply
increasing the bit size of the representation of the end-points.

\subsubsection{Dependency}
\label{subsubsec:dependency}

By default, interval computations do not take into account any
dependencies between parameters. For instance, assume that
$x \in [0,1]$. By applying interval subtraction we obtain
$ (x-x) \in [-1,1]$, which is a significant overestimation as $x-x=0$,
regardless of which interval $x$ belongs to.

\begin{figure}[t]
  \centering
   \scalebox{0.3}[0.3]{\includegraphics{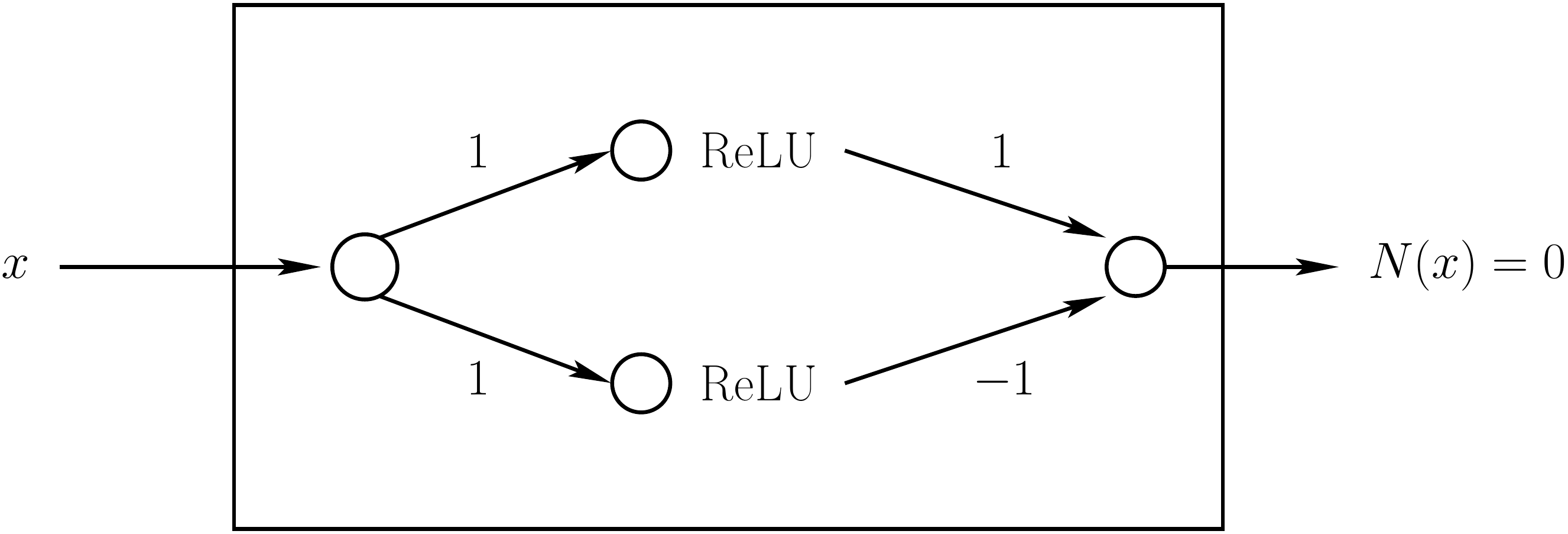}}
   \caption{Dependency problem: The network computes the constant
     function zero. Hence, we must have $N'(x)=0$ for all $x$. But the
     Clarke-gradient at $x=0$ is overestimated as $[-1,1]$.}
  \label{fig:dependency-zero_NN}
\end{figure}

The dependency problem causes inaccuracies in estimations of the
gradient as well. As an example, consider the neural network of
Figure~\ref{fig:dependency-zero_NN}. This network computes the
constant function $N(x)=0$, for which we must have
$\forall x \in \R: N'(x)=0$. We estimate the Clarke-gradient of $N$
using the chain rule, which, at $x = 0$, leads to the following
overestimation:
\begin{equation*}
\clarke{N}(0) \subseteq \clarke{\ReLU}(0) - \clarke{\ReLU}(0) = [0,1] -
[0,1] =  [-1,1].
\end{equation*}

Our framework, however, is not affected by the dependency problem for
general position $\ReLU$ networks (Theorem~\ref{thm:gen_pos_ReLU}) or
differentiable networks (Theorem~\ref{thm:differentiable_networks}).

\subsubsection{Wrapping Effect}
\label{subsubsec:wrapping_effect}

Whenever a set which is not a hyperbox is overapproximated by a
bounding hyperbox, some accuracy is lost. This source of
inaccuracy in interval arithmetic is referred to as the wrapping
effect~\parencite{Neumaier:Wrapping_Effect:1993}. As stated in
Proposition~\ref{prop:Clarke_general_properties}, the Clarke-gradient
is, in general, a convex set. Hence, by overapproximating the
Clarke-gradient with a hyperbox, some accuracy will be lost, as
exemplified in Example~\ref{example:L_hat_overapprox_L_derive}.

Similar to the dependency problem, our framework is not affected by
the wrapping effect for general position $\ReLU$ networks
(Theorem~\ref{thm:gen_pos_ReLU}) and differentiable networks
(Theorem~\ref{thm:differentiable_networks}).

 
\subsection{General Position \texorpdfstring{$\boldsymbol{\ReLU}$}{ReLU} Networks}
\label{subsec:general_position_networks}

We have pointed out the trade-off between accuracy and efficiency in
using the hyperbox domain~$\intvaldom[\R^n\lift]$, and further sources
of inaccuracy in interval arithmetic. In this section, we show that,
for a broad class of $\ReLU$ networks, no accuracy is lost by using
hyperboxes. Assume that $N : \R^n \to \R$ is a feedforward $\ReLU$
network with $k \geq 0$ hidden neurons (and an arbitrary number of
hidden layers). For convenience, we adopt the set-theoretic definition
of natural numbers~\parencite{Jech2002} for which we have
$k = \set{0, \ldots, k-1}$. For each $i \in k$, let
$\hat{z}_i : \R^n \to \R$ be the pre-activation function of the
$i$-th hidden neuron and define:
\begin{equation}
  \label{eq:def_H_i}
  H_i \defeq \setbarNormal{x \in \R^n}{\hat{z}_i(x) = 0}.
\end{equation}
\noindent
The following
definition, which we take
from~\parencite{Jordan_Dimakis:Exactly_NeurIPS:2020}, is an adaptation of
the concept of hyperplane arrangements in general
position~\parencite{Stanley:Hyperplane_Arrangements:Book:2006}:
\begin{definition}
\label{def:ReLU_General_Position}
  A $\ReLU$ network $N : \R^n \to \R$ with $k \geq 0$ hidden neurons
  is said to be in \emph{general position} if, for every subset
  $S \subseteq k$, the intersection $\bigcap_{i \in S} H_i$ is a
  finite union of $(n - \absn{S})$-dimensional polytopes.
\end{definition}

A $\ReLU$ network $N : \R^n \to \R$ with $k \geq 0$ hidden neurons and
only one hidden layer is in general position if for every
$S \subseteq k$, the intersection $\bigcap_{i \in S} H_i$ is an
$(n - \absn{S})$-dimensional subspace of $\R^n$. For instance, when
$n=2$, the network is in general position if no two hyperplanes $H_i$
and $H_j$ (with $i \neq j$) are parallel and no three such hyperplanes
meet at a point.

Recall that
$\tilde{L}(N): \intvaldom[\R^n\lift] \to \intvaldom[\R^n\lift]$
denotes the interval approximation of the $\hat{L}$-derivative of $N$
obtained by applying the chain rule of
Lemma~\ref{lem:chain_Rule_hat_L}. In general, we have:
  \begin{equation}
    \label{eq:Clarke_subsets}
    \forall x \in \R^n:  \quad \clarke{N}(x) \subseteq \hat{L}(N)(x) \subseteq
    \tilde{L}(N)(\set{x}). 
  \end{equation}
  \noindent
  Over general position networks, however, we obtain equality:  
\begin{lemma}
  \label{lemma:no_loss_of_precision_gen_pos}
  Let $N : \R^n \to \R$ be a general position $\ReLU$
  network. Then, $\forall x \in \R^n: \clarke{N}(x) = \hat{L}(N)(x) =
  \tilde{L}(N)(\set{x})$.
\end{lemma}
\begin{proof}  
  The fact that
  $\forall x \in \R^n:\clarke{N}(x) = \tilde{L}(N)(\set{x})$
  follows
  from~\parencite[Theorem~2]{Jordan_Dimakis:Exactly_NeurIPS:2020}. The
  full claim now follows from~\eqref{eq:Clarke_subsets}.
\end{proof}

   As a result, we have an effective procedure which, for any given
   $\ReLU$ network $N : [-M,M]^n \to \R$ in general position, and any
   $X \in B_{\kCPPO{[-M,M]^n}}$, returns the Lipschitz constant of $N$
   over $X$, and the set of points in $X$ where the maximum Lipschitz
   values are attained, to within any degree of accuracy:

\begin{theorem}
  \label{thm:gen_pos_ReLU}
  Let $N : [-M,M]^n \to \R$ be a general position $\ReLU$ network and
  $X \in B_{\kCPPO{[-M,M]^n}}$. Then:
  \begin{equation*}
    \left\{
      \arrayoptions{0.5ex}{1.2}
      \begin{array}{lcl}
        \maxV(\clarke{N},X) &= &\maxV(\tilde{L}(N), X),\\
        \maxS(\clarke{N},X) &= &\maxS(\tilde{L}(N), X),
      \end{array}
    \right.
  \end{equation*}
  \noindent
  and both the maximum value and the maximum set can be obtained
  effectively to any degree of accuracy.
\end{theorem}

\begin{proof}
  The result follows from
  Lemma~\ref{lemma:no_loss_of_precision_gen_pos},
  Theorem~\ref{thm:completeness}, and Theorem~\ref{thm:computability}.
\end{proof}

The network of Figure~\ref{fig:dependency-zero_NN} is not in general
position, and we have shown that
Lemma~\ref{lemma:no_loss_of_precision_gen_pos} fails on that
network. With respect to the Lebesgue measure over the parameter
space, however, almost every $\ReLU$ network is in general
position~\parencite[Theorem~4]{Jordan_Dimakis:Exactly_NeurIPS:2020}. In
simple terms, \emph{we have a sound and effective procedure for
  Lipschitz analysis of networks, which provides the results to within
  any given degree of accuracy on almost every $\ReLU$ network}.

 
\subsection{Differentiable Networks}
\label{subsec:differentiable_networks}

One of the main strengths of our domain-theoretic framework is that it
can handle differentiable and non-differentiable networks alike
seamlessly. For differentiable networks as well, we obtain results
similar to those for $\ReLU$ networks in general position.

We know that when $N$ is classically differentiable at
$x \in [-M,M]^n$, the $L$-derivative and the classical derivative
coincide~\parencite{Edalat:2008:Continuous_Derivative}. Furthermore, for
classically differentiable functions, the chain rules of
Proposition~\ref{prop:Clarke_chain_rule} and
Lemma~\ref{lem:chain_Rule_hat_L} reduce to equalities which incur no
overapproximation, because both sides are single-valued. As a result,
we obtain:
\begin{lemma}
  \label{lemma:no_loss_of_precision_differentiable}
  Let $N : [-M,M]^n \to \R$ be a differentiable feedforward network in
  ${\cal F}$. Then:
  \begin{equation*}
    \forall x \in [-M,M]^n:\quad \clarke{N}(x) = \hat{L}(N)(x) =
    \tilde{L}(N)(\set{x}). 
  \end{equation*}
  \noindent
  In other words, $\hat{L}(N)$ is an interval extension of
  $\clarke{N}$.
\end{lemma}

\begin{theorem}
  \label{thm:differentiable_networks}
  Let $N : [-M,M]^n \to \R$ be a differentiable feedforward network in
  ${\cal F}$. Then:
  \begin{equation*}
    \left\{
      \arrayoptions{0.5ex}{1.2}
      \begin{array}{lcl}
        \maxV(\clarke{N},X) &= &\maxV(\tilde{L}(N), X),\\
        \maxS(\clarke{N},X) &= &\maxS(\tilde{L}(N), X),
      \end{array}
    \right.
  \end{equation*}
  \noindent
  and both the maximum value and the maximum set can be obtained
  effectively to any degree of accuracy.
\end{theorem}



\section{Experiments}
\label{sec:experiments}

We have implemented our method for estimation of Lipschitz constant of
neural networks---which we call \emph{lipDT}---using two interval
libraries: the arbitrary-precision C++ library
MPFI~\parencite{Revol_Rouillier:MPFI:05}, and the fixed precision
Python library
PyInterval.\footnote{\url{https://github.com/taschini/pyinterval}} The
latter is used for fair time comparison with the other methods in
Tables~\ref{table:Experiment1} and \ref{table:mnist}, which are all in
Python, {\ie}, ZLip~\parencite{Jordan_Dimakis:Provable_ICML:2021},
lipMIP~\parencite{Jordan_Dimakis:Exactly_NeurIPS:2020},
FastLip~\parencite{Weng_et_al:Towards_FastLip:2018},
SeqLip~\parencite{Virmaux_Scaman:Lipschitz_regularity:2018}, and
\ac{CLEVER}~\parencite{Weng_et_al-CLEVER-ICLR:2018}. All the
experiments have been conducted on a machine with an AMD Ryzen 75800H
3.2GHz 8-Core processor and 16GB RAM under Ubuntu 20.04 LTS
environment. The source code is available on
GitHub.\footnote{\url{https://github.com/Can-ZHOU/Theoretic-Robustness-NN}}

The main characteristic of our method is that it is fully
validated. We begin by demonstrating how our method can be used for
verification of state of the art methods from the literature through
some experiments on very small networks. If the state of the art
methods turn out to be prone to inaccuracies over small networks, then
they most likely are prone to errors over large networks as well.

Whenever a network has $m$ output neurons, we only consider the first
output neuron for analyses. We need to consider networks with more
than one output to avoid altering the methods from the literature which
require more than one output. For instance, in Experiment~1, we
consider a randomly generated $\ReLU$ network of layer sizes
$[2,2,2]$, with an input domain straddling the intersection of the
hyperplanes $H_1$ and $H_2$, as defined in~\eqref{eq:def_H_i}.

\begin{table*}[t]
\centering
\caption{Experiment 1: Random, $\ReLU$, Layer sizes = [2, 2, 2]. Our method is called lipDT, which is the only case where
    both the method and the implementations are validated.}
\begin{tabular}{|c|c|c|c|c|}
\hline
Guarantee & Method                & Estimate of Lip. Constant                               & Time (ms)
  & Rel. Err. \\ \hline \hline
\multirow{4}{*}{Fully validated} & \textbf{lipDT (C++)}  &  \makecell*[c]{[0.32270569085905976,\\
  0.32270569085905976]} & 8.2890 & \textbf{0.00\%}    \\ \cline{2-5}
& \textbf{lipDT (Python)}  &  \makecell*[c]{[0.32270569085905976,\\
  0.32270569085905976]} & 338.3580 & \textbf{0.00\%}    \\ \hline  \hline
  \multirow{3}{*}{\parbox{3cm}{Method validated, implementation
  non-validated}}  &   
ZLip      & 0.041121020913124084               & 2.5075   &
                                                                   --87.26\% \\ \cline{2-5}
  & lipMIP            & 0.041121020913124084               & 10.0348     & --87.26\%\\ \cline{2-5}
  & FastLip         & 0.041121020913124084				 &
                                                                   0.8451
  & --87.26\% \\ \hline \hline
\multirow{2}{*}{Non-validated} &  SeqLip          & 0.2129872590303421
                                                     & 1.9632 &
                                                                --34.00\%
  \\ \cline{2-5}
  & CLEVER          & 0.2815846800804138				 & 1054.1022 & --12.74\%\\ \hline
\end{tabular}
\label{table:Experiment1}
\end{table*}

Assume that $z = [\lep[z],\uep[z]]$ is the interval returned by
lipDT. For each of the other methods, if $r$ is the returned estimate
of the Lipschitz constant, we calculate the relative error as follows:
\begin{equation}
  \label{eq:rel_error}
  \text{relative error} \defeq \left\{
    \begin{array}{ll}
      \arrayoptions{0.5ex}{1.3}
      \frac{r - \lep[z]}{\lep[z]}, & \text{if } r < \lep[z],\\
      0, & \text{if } \lep[z] \leq r \leq \uep[z],\\      
      \frac{r - \uep[z]}{\uep[z]},& \text{if } \uep[z] < r.
    \end{array}
  \right.
\end{equation}
\noindent
As such, if the relative error is negative, then the value $r$ is
definitely an underapproximation of the true Lipschitz constant, and
when the relative error is positive, then $r$ is definitely an
overapproximation of the true Lipschitz constant. As reported in
Tables~\ref{table:Experiment1} and \ref{table:mnist}, although lipMIP,
ZLip, and FastLip are based on validated methods, the implementations
are not validated due to floating-point errors. In
Table~\ref{table:Experiment1}, the estimates of the Lipschitz constant
returned by all the three methods are exactly the same, which
demonstrates that, for the input domain of Experiment 1, these methods
all follow the same computation path, which is subject to the same
floating-point errors.

For the randomly generated network of Experiment 1, the weights and
biases are shown in Figure~\ref{fig:weights_and_biases_Exp_01}. Based
on these values, we calculate an input domain which straddles the
intersection of the hyperplanes $H_1$ and $H_2$. The center and radius
of the input domain are also given in
Figure~\ref{fig:weights_and_biases_Exp_01}.
\begin{figure*}[t]
  \centering
  Weights $W_1$ from input neurons to the hidden layer:
\begin{equation*}
  W_1 =
  \begin{bmatrix}
    -0.5323450565338134766 & 0.2982044517993927002\\
    -0.4367777407169342041 & 0.1963265985250473022
  \end{bmatrix}  
\end{equation*}

Biases $b_1$ of the hidden layer:
\begin{equation*}
  b_1 =
  \begin{bmatrix}
  -0.3763356208801269531&  -0.6647928357124328613  
  \end{bmatrix}
\end{equation*}

Weights $W_2$ from the hidden layer to the output layer:
\begin{equation*}
  W_2 =
  \begin{bmatrix}
    -0.3885447978973388672 & 0.4447681903839111328\\
   -0.3205857872962951660 & 0.2034754604101181030
  \end{bmatrix}  
\end{equation*}

Input domain:
\begin{equation*}
\text{center} = [x_0,x_1] = [-4.832202221268014242,
-7.364287590384273940], \quad \text{radius} = 10^{-7}
\end{equation*}

\caption{Weights, biases, and the input domain of the network of
  Experiment 1.}
  \label{fig:weights_and_biases_Exp_01}
\end{figure*}

For lipDT to produce the results of Experiment 1, it had to perform
five iterations of the main loop of
Algorithm~\ref{alg:interval_max}. In each iteration, the most
important operation is that of the bisection, carried out in
line~\ref{alg:line:bisect} of
Algorithm~\ref{alg:interval_max}. Table~\ref{tab:exp_01_bisections}
shows the dynamics of the computation through successive bisections.

\label{page:lipDT_bisection}

\begin{table*}
  \centering
  \begin{tabular}{|c|c|} \hline
    Number of Bisections & Interval enclosure of the Lipschitz constant \\ \hline \hline
    0 & [0.0, 0.32270569085905976]\\ \hline
    1 & [0.04112101957020187, 0.32270569085905976]\\ \hline
    2 & [0.04112101957020187, 0.32270569085905976] \\ \hline
    3 & [0.04112101957020187, 0.32270569085905976]\\ \hline
    4 &  [0.04112101957020187, 0.32270569085905976]\\ \hline
    5 & [0.32270569085905976, 0.32270569085905976]    \\ \hline
  \end{tabular}
  \caption[Results of each bisection]{The results of lipDT on
    Experiment 1 for the first 5 bisections.}
  \label{tab:exp_01_bisections}
\end{table*}

The effects of the floating-point errors become less significant as
the radius of the input domain is increased, as demonstrated in
Figure~\ref{fig:accuracy_comparison_radius}. We took the same center
as given in Figure~\ref{fig:weights_and_biases_Exp_01}, but changed
the radius of the input domain from $10^{-7}$ to $1$, and compared the
results returned by all the methods. ZLip, lipMIP, and FastLip give
exactly the same results. Hence, we only draw the results for ZLip,
SeqLip, \ac{CLEVER}, and lipDT. As can be seen, \ac{CLEVER} and
ZLip---and as a consequence, the other methods except for SeqLip---all
catch up with lipDT eventually when the radius of the domain is large
enough.

\begin{figure}[t]
  	\centering
	\scalebox{0.2}[0.18]{\includegraphics{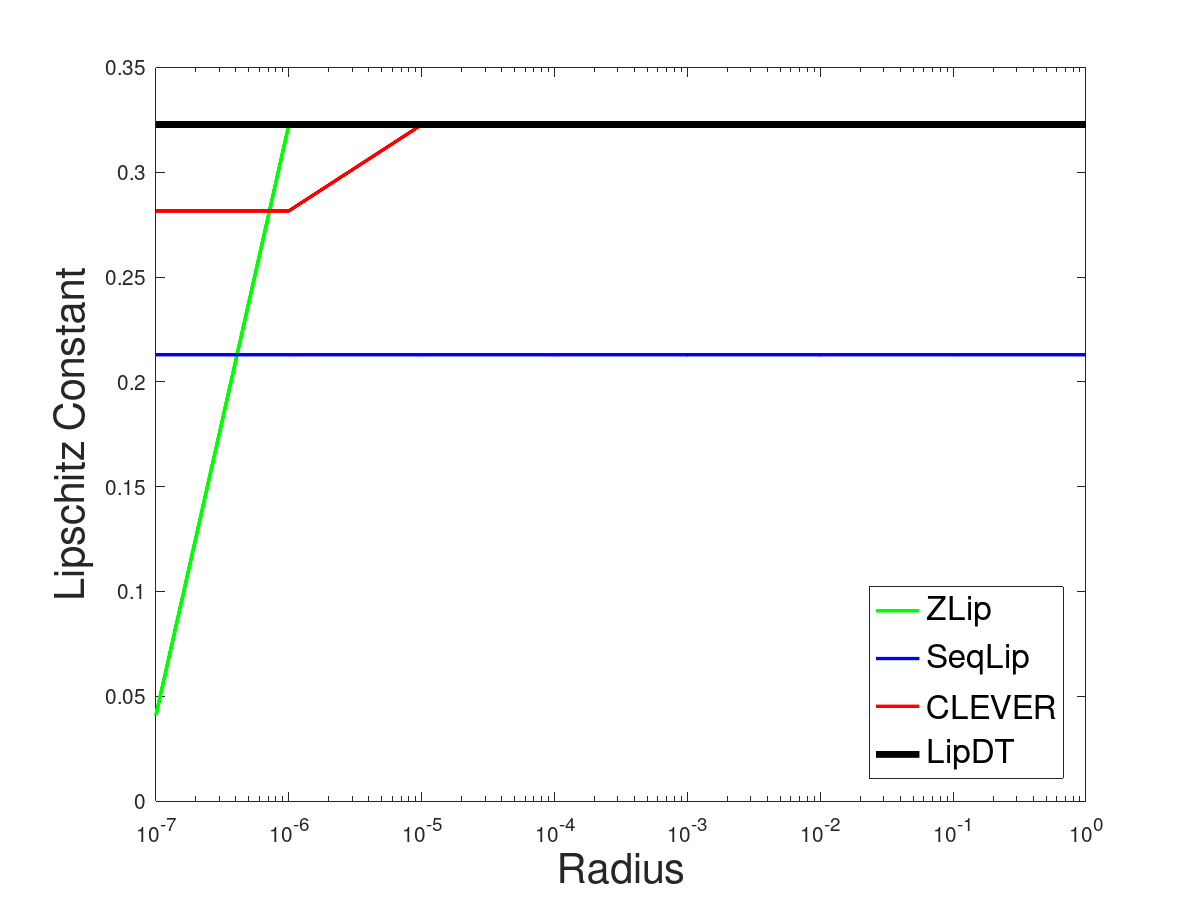}} 
	\caption{For very small input domains, lipDT is the only
          method which returns the correct result.}
	\label{fig:accuracy_comparison_radius}
\end{figure}

\subsection{Scalability: input dimension}
\label{subsec:scalability:input_dim}

When no bisections are required, lipDT can be useful for analysis of
larger networks as well, as shown in Experiment 2 on the MNIST dataset
(Table~\ref{table:mnist}). We consider a $\ReLU$ network of size
$[784, 10, 10, 2]$, only focusing on the digits $1$ and $7$ in the
output. The radius of the input domain is $0.001$. For this
experiment, the input domain is not situated on the intersection of
the hyperplanes, because lipDT would require bisecting along 784
dimensions, resulting in $2^{784}$ new hyperboxes in each bisection,
which is prohibitively costly. Nonetheless, we observe that, even
without any bisection, lipDT maintains correctness, while ZLip,
lipMIP, and FastLip still underestimate the Lipschitz constant, albeit
with small relative errors.

\begin{table*}[t]
\centering
\caption{Experiment 2: MNIST, $\ReLU$, Layer sizes = [784, 10, 10, 2].}
\begin{tabular}{|c|c|c|c|c|}
\hline
Guarantee & Method                & Estimate of Lip. Constant                                     &
                                                                    Time (ms)      & Rel. Err.\\ \hline \hline
\multirow{4}{*}{Fully validated} & \textbf{lipDT C++} & \makecell*[c]{[0.34815343269413546,\\ 0.34815343269413546]}  & 164.0400 & \textbf{0.00\%}\\ \cline{2-5}
& \textbf{lipDT Python} & \makecell*[c]{[0.34815343269412385,\\
  0.3481534326941471]}  & 2388.6251 & \textbf{0.00\%}\\ \hline \hline
    \multirow{3}{*}{\parbox{3cm}{Method validated, implementation
  non-validated}}  &  ZLip          & 0.34815341234207153
                                                                              &
                                                                                7.1895   & --5.85e-06\%\\ \cline{2-5}
  & lipMIP            & 0.3481534055299562                       & 33.1788
                                                                               &
                                                                                 --7.80e-06\% \\ \cline{2-5}
  & 
FastLip         & 0.34815341234207153
                                                                  &
                                                                    0.8844
                                                                               & --5.85e-06\%\\ \hline \hline
  \multirow{2}{*}{Non-validated} &
SeqLip          & 0.670467323694055
                                                                              &
                                                                                4.9305    & +92.58\%\\ \cline{2-5}
  & CLEVER          & 0.3481535315513611						& 4308.8220 & +2.84e-05\%\\ \hline
\end{tabular}
\label{table:mnist}
\end{table*}

For a network with input dimension $n$, each bisection generates $2^n$
new hyperboxes. As such, under favorable conditions, our method is
applicable to networks with input dimensions up to 10 or so. It is
also possible to obtain more accurate results on networks with higher
input dimensions using \emph{influence
  analysis}~\parencite{Wang_et_al:symbolic_intervals:2018}. This can
be achieved by a slight modification of
Algorithm~\ref{alg:interval_max}. The influence analysis technique,
however, is only applicable on differentiable networks. Furthermore,
this technique provides tighter bounds on the Lipschitz constant in a
reasonable time, but may still provide a loose bound around the
maximum set, {\ie}, the set where maximum Lipschitz constant is
attained. Another simple way of increasing efficiency is through
parallelization. The reason is that the algorithm is essentially
performing a search. Once several hyperboxes are generated, they can
be searched in a perfectly parallel way.

\subsection{Scalability: depth and width}
\label{subsec:scalability:depth_width}

As we have mentioned previously, validated methods of Lipschitz
estimation (such as lipDT) cannot be scalable. Nonetheless, such
methods are valuable for purposes such as:

\begin{enumerate}[label=(\roman*)]
\item verficiation of other Lipschitz estimation methods, as
  demonstrated in Table~\ref{table:Experiment1};

\item robustness analysis of low-dimensional networks, especially
  those that appear in safety-critical applications, {\eg}, in power
  electronics~\parencite{Dragicevic_et_al:AI_Aided_PowElec:2019},
  analog-to-digital converters~\parencite{Tankimanova:NN_ADC:2018}, and
  solution of \acp{PDE}~\parencite{Rudd:PDE_NN:PhD_Thesis:2013};

\item \label{item:lipDT_under_over_param} investigation of how
  Lipschitz constant is affected by factors such as
  under-parametrization, over-parametrization, and through successive
  epochs of a training process.
  
\end{enumerate}

Item~\ref{item:lipDT_under_over_param} is particularly relevant for
networks with low input dimensions, but with depths and widths that
can be very large relative to their input dimensions. The reasons is
that, although lipDT does not scale well with input dimension of the
network, it has a much better scalability with respect to the depth
and width of the network.

To investigate how lipDT scales with the number of layers ({\ie},
depth) of the network, we generated some multilayer networks randomly
with 4 input neurons, 2 output neurons, and various number of hidden
layers (from 1 to 10), each containing 10 neurons. We took the input
domain with center $[1,1,1,1]$ and radius $0.001$, and compared the
timing of lipDT with the other methods, except for \acs{CLEVER}, which
takes significantly longer than all the other methods. As can be seen
in Figure~\ref{fig:scalability}, lipDT scales quite well---almost
linearly---with the number of layers of the network.

\begin{figure*}[t]
  \centering
  \scalebox{0.6}[0.7]{\includegraphics{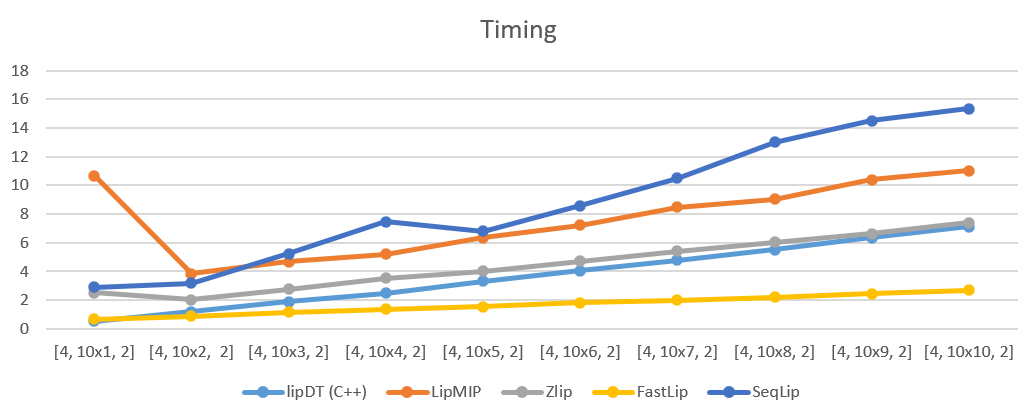}} 
  \caption{lipDT scales (almost) linearly with respect to the number
    of layers.}
  \label{fig:scalability}
\end{figure*}

Similarly, lipDT exhibits an almost linear scaling with respect to the
width of the network. We ran an experiment in which we considered
networks with two input neurons and one hidden layer of width $2^n$,
where $n$ ranges from $1$ to $10$. The weights and biases were
assigned randomly, and to make the comparison fair, we performed the
same number of bisections for each experiment, making sure that the
only contributing factor is the width. We recorded the time, and the
result is shown in Figure~\ref{fig:time_vs_width}.

\begin{figure*}[t]
  \centering
  \scalebox{0.55}[0.5]{\includegraphics{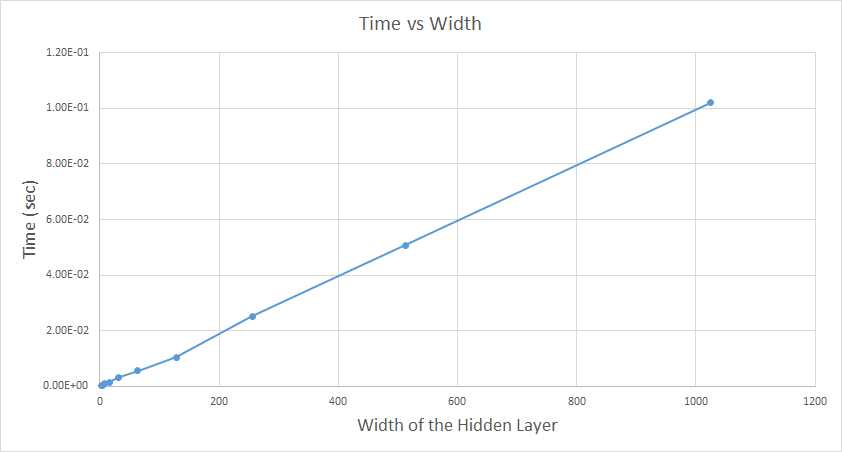}}
  \caption{lipDT scales almost linearly with respect to the width of
    the network.}
  \label{fig:time_vs_width}
\end{figure*}

\subsection{Case study: benign overfitting}
\label{subsec:Case_study_Benign_overfitting}

The traditional consensus in machine learning postulated that
over-parametrized models would have poor generalization
properties. For instance, in a regression task over a dataset of size
$m$, if the capacity $K$ of the model were significantly larger than
$m$, then the model would interpolate the dataset, and would oscillate
erratically over points outside of the dataset. In a fairly recent
development, it has been observed that over-parametrized models can
indeed exhibit superior generalization compared with
under-parametrized ones, even over noisy
data~\parencite{Muthukumar_et_al:Harmless_Interpolation:2020}. An
extensive account of this phenomenon---termed \emph{`benign
  overfitting'}~\parencite{Bartlett_et_al:Benign_overfitting:2020}---can
be found in~\parencite{Belkin:Fit_without_fear:2021}.

In this section, we investigate benign overfitting from the angle of
Lipschitz analysis. To that end, we consider a univariate regression
task similar to that depicted
in~\parencite[Figure~3.6]{Belkin:Fit_without_fear:2021}. The dataset
consists of eleven points over the interval $[-5,5]$. We train a set
of $\ReLU$ networks with one input neuron, one output neuron, and a
hidden layer with width ranging from $3$ to $3000$ (in steps of
$3$). Even though lipDT scales well with respect to the depth of the
network, for this experiment, we consider only networks with one
hidden layer, and take the width of the single hidden layer as the
capacity of the network.

Each network is trained for $100\,000$ epochs. Then, the Lipschitz
constant is calculated over the interval $[-5,5]$ using lipDT. The
result is shown in Figure~\ref{fig:Lip_vs_capacity_200} (Left). To
show the initial dynamics more clearly, the right hand figure shows
capacities only up to $200$. The curve exhibits similar qualitative
dynamics as that of the ``double descent generalization curve''
of~\parencite[Figure~3.5]{Belkin:Fit_without_fear:2021}. The
under-parametrized regime extends from capacity $1$ up to capacity
$42$, during which the Lipschitz constant has an increasing and then
decreasing trend. Maximum Lipschitz constant is attained at capacity
$21$. Interpolation starts at capacity $42$, after which the Lipschitz
constant hovers around $4$ until it flattens after capacity $125$, and
remains flat for the remaining capacities up to $3000$.

\begin{figure*}[t]
  \centering
  \scalebox{0.45}[0.40]{\includegraphics{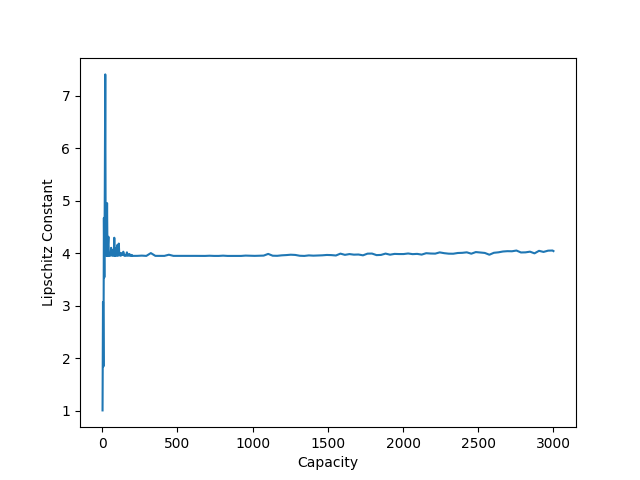}}    
  \scalebox{0.45}[0.40]{\includegraphics{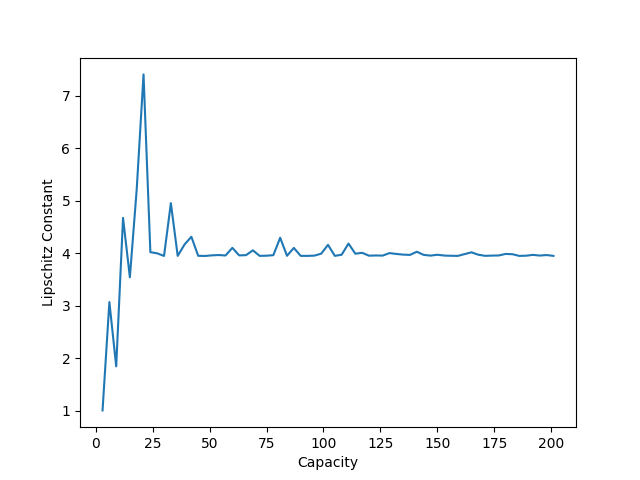}}
  \caption{Lipschitz constant versus capacity. Left: All the
    capacities up to $3000$. Right: Capacities up to $200$. The
    under-parametrized regime extends to capacity $42$, where
    interpolation starts, and after which, the curve remains almost
    flat. Maximum Lipschitz constant is attained at capacity $21$.}
  \label{fig:Lip_vs_capacity_200}
\end{figure*}

Figure~\ref{fig:under_over_param_curves} depicts the output of neural
networks with capacities $3$ and $21$ (under-parametrized), $42$
(interpolation threshold) and $200$ (over-parametrized). The slope of
the linear segment connecting the third and fourth data points (from
left) is roughly $4$. By the mean value theorem, it is clear that once
interpolation starts, the Lipschitz constant may not go below this
slope. As can be seen in Figure~\ref{fig:Lip_vs_capacity_200}, in the
over-parametrized regime, the Lipschitz constant stays very close to
$4$.

We also investigate how the Lipschitz constant is affected through
training. The results are shown in
Figure~\ref{fig:training_Lipschitz}. In particular, when the capacity
is farther away from the interpolation threshold---{\eg}, for
capacities $3$ and $200$---the Lipschitz curve flattens at earlier
epochs. In contrast, for capacities $21$ and $42$, optimization seems
a slower and more demanding process, as the curve begins to flatten at
much later epochs. We point out, however, that the optimization
algorithm is randomized, and the network is initialized
randomly. Hence, the exact dynamics of the Lipschitz constant versus
epoch number---in particular, at what stage the curve begins to
flatten---are subject to the initial weights and biases and the
optimization process.

\begin{figure*}[t]
  \centering
  \scalebox{0.4}[0.3]{\includegraphics{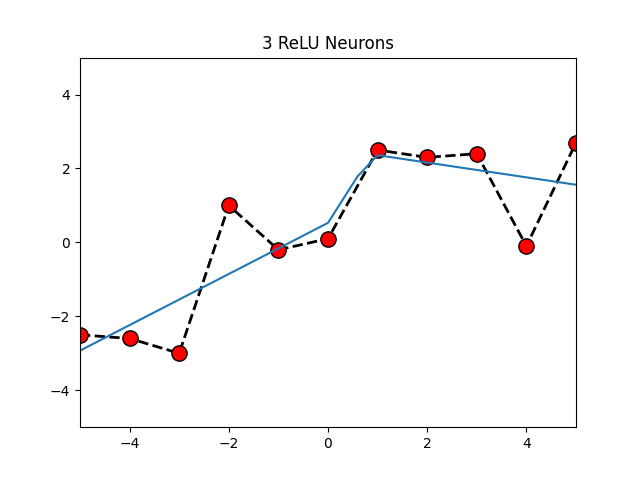}}
  \scalebox{0.4}[0.3]{\includegraphics{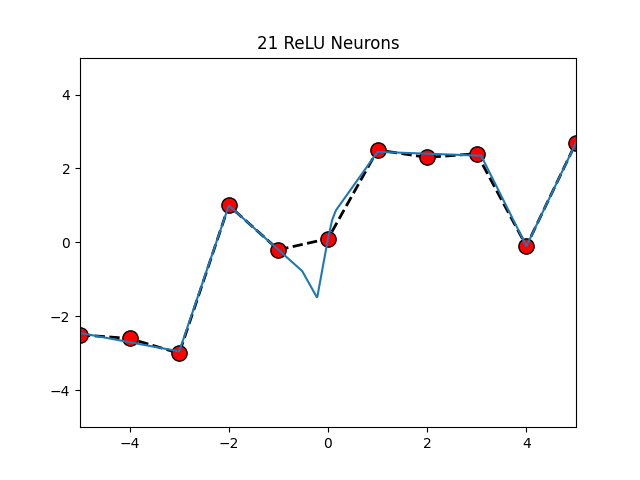}}
  \scalebox{0.4}[0.3]{\includegraphics{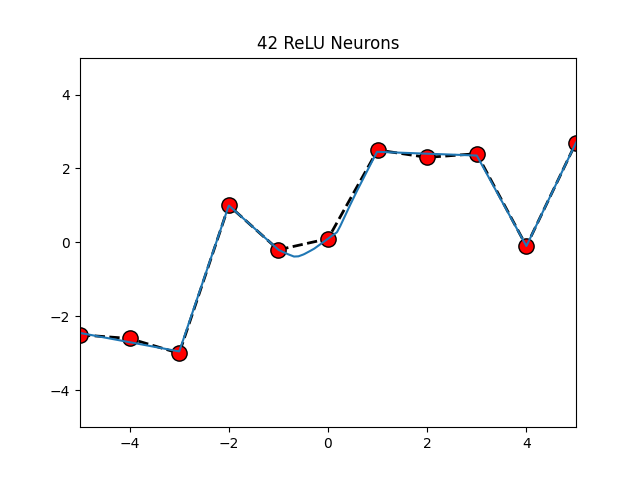}}
  \scalebox{0.4}[0.3]{\includegraphics{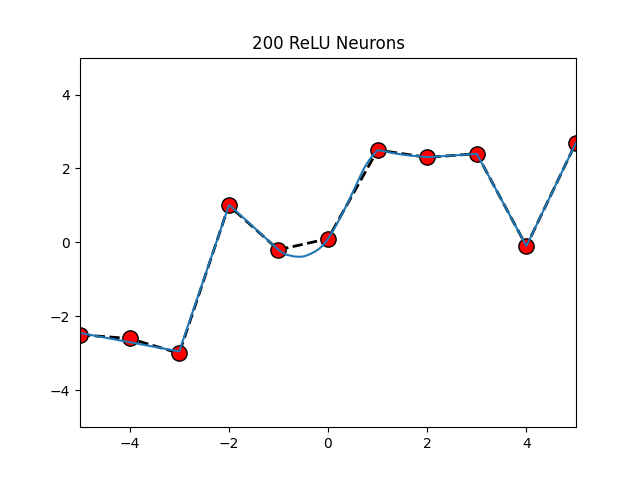}}
  \caption{The dashed black curve is a piecewise linear interpolation
    of the data points. The blue curves show the output of the neural
    networks. Interpolation starts at capacity $42$. The highest
    Lipschitz constant is attained at capacity $21$. Also, see,~\parencite[Figure~3.6]{Belkin:Fit_without_fear:2021}.}
  \label{fig:under_over_param_curves}
\end{figure*}

\begin{figure*}[t]
  \centering
  \scalebox{0.4}[0.3]{\includegraphics{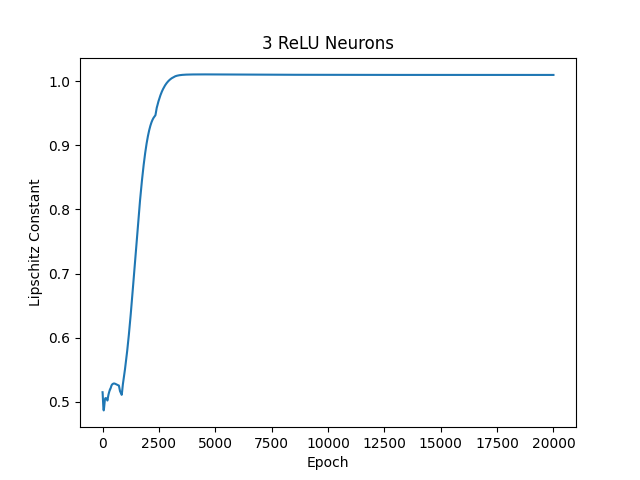}}
  \scalebox{0.4}[0.3]{\includegraphics{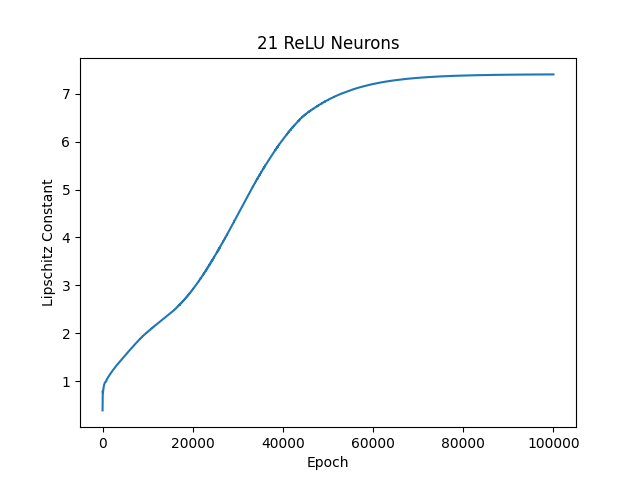}}
  \scalebox{0.4}[0.3]{\includegraphics{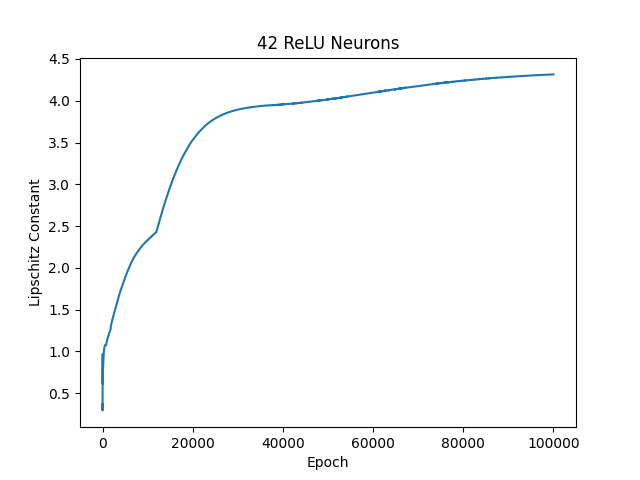}}
  \scalebox{0.4}[0.3]{\includegraphics{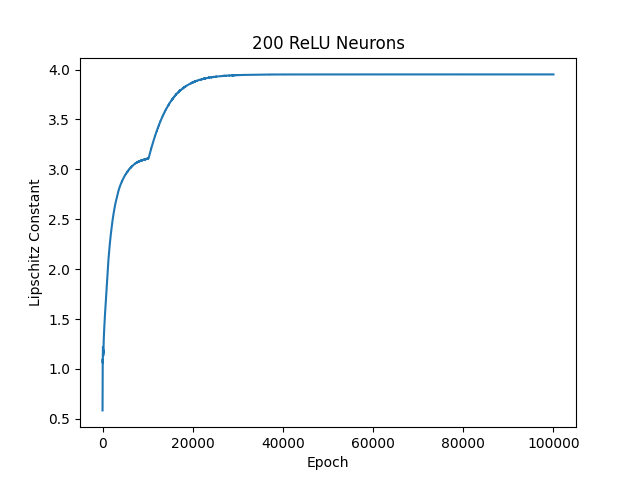}}  
  \caption{Dynamics of the Lipschitz constant through successive
    epochs. The network with three neurons reaches its best loss much
    earlier than $100\,000$ epochs, so we stop the training at
    $20000$ epochs.}
  \label{fig:training_Lipschitz}
\end{figure*}


\subsubsection{Validated versus non-validated Lipschitz analysis}

We ran our experiments using the validated method lipDT, and the
non-validated statistical method \ac{CLEVER}. The Lipschitz estimates
returned by \ac{CLEVER} are almost always a slight under-estimate of
those returned by lipDT. To be more precise, the average relative
error of \ac{CLEVER} is around $-1\%$, calculated according
to~\eqref{eq:rel_error}. In most cases, qualitatively, the results of
\ac{CLEVER} and lipDT are similar. Some outliers, however, were
observed. For instance,
Figure~\ref{fig:benign_clever_wrong_3000}~(Left) shows the graph of
the result returned by \ac{CLEVER} when measuring the Lipschitz
constant of the $\ReLU$ network with $3000$ neurons, up to epoch
$2500$. At epoch $251$, \ac{CLEVER} estimated the Lipschitz constant
to be around $127$, while for all the other epochs, the Lipschitz
constant was estimated below~$4$.

\begin{figure}[t]
  \centering
  \scalebox{0.45}[0.4]{\includegraphics{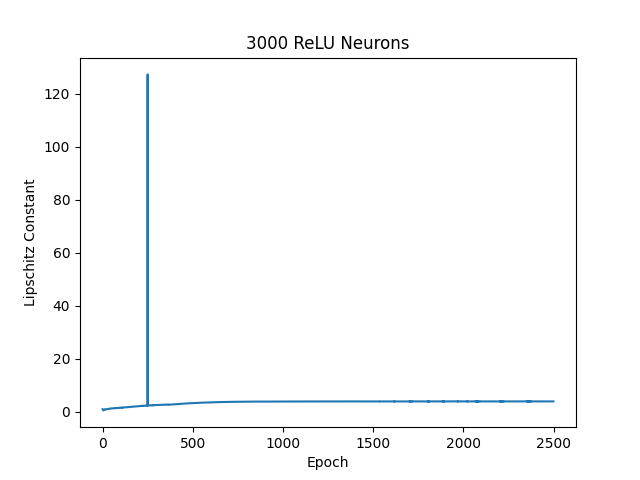}}
  \scalebox{0.45}[0.4]{\includegraphics{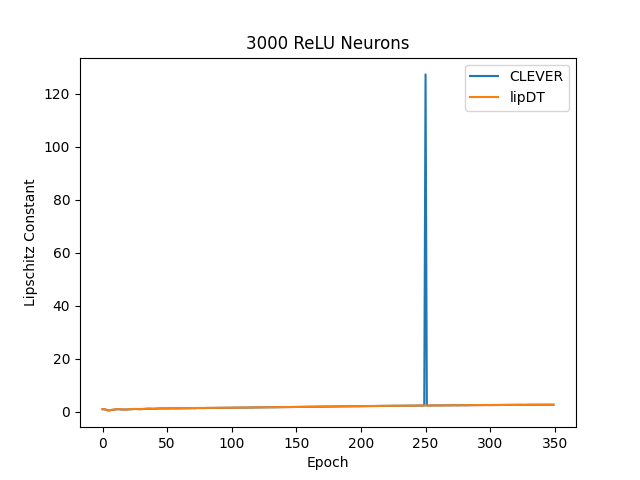}}
  \caption{Left: \ac{CLEVER} returns an outlier at epoch $251$. Right:
    Comparison with lipDT (for the first $350$ epochs) reveals that
    the Lipschitz constant does not have a jump.}
  \label{fig:benign_clever_wrong_3000}
\end{figure}

Without recourse to a validated method, it is not possible to
determine whether the Lipschitz constant truly has a jump at epoch
$251$ or not. The reason is that the commonly used optimization
methods are randomized, which makes it possible to obtain genuine
outliers at certain epochs. Nonetheless, running the same experiment
with lipDT reveals that, at epoch $251$, \ac{CLEVER} indeed provides
the wrong result. The interval returned by lipDT at epoch $251$ is:
\begin{equation*}
[2.445572934050687,2.445572934050933].   
\end{equation*}
Figure~\ref{fig:benign_clever_wrong_3000}~(Right)
depicts the results (for the first $350$ epochs) of the comparison
between lipDT and \ac{CLEVER}.

While in experiments of Table~\ref{table:Experiment1} and
Figure~\ref{fig:accuracy_comparison_radius}, the input domains were
chosen deliberately to be small and to straddle the intersection of
hyperplanes, in the experiment of
Figure~\ref{fig:benign_clever_wrong_3000}, the search was performed
over the entire interval of $[-5,5]$. This demonstrates that
non-validated statistical methods may return misleading results even
when the samples are taken from a large and seemingly non-critical
part of a very low-dimensional input space.


\section{Concluding remarks}
\label{sec:concluding_remarks}

Continuous domains are powerful models of computation and provide a
semantic model for the concept of approximation. We have demonstrated
this claim by presenting a framework for robustness analysis of neural
networks. The main focus in the current article has been laying the
theoretical foundation and obtaining fundamental results on soundness,
completeness, and computability. We also demonstrated how
straightforward it is to translate our algorithms into fully validated
implementations using interval arithmetic. In particular, we devised
an algorithm for validated computation of Lipschitz constants of
feedforward networks, which is sound and complete. Based on the
foundation laid in this paper, our next step will be extending the
framework to \acp{RNN}.

The maximization algorithm of Algorithm~\ref{alg:interval_max} has an
imperative style, which is in line with the common practice in machine
learning. Continuous domains, however, are best known as models of
lambda calculus and it is common to adopt a functional style in
language design for real number
computation~\parencite{Escardo96-tcs,Farjudian:Shrad:2007,DiGianantonio_Edalat-PCDF:2013}. Thus,
another direction for future work is the design of a functional
algorithm for maximization which suits our non-algebraic domain
framework.



\printbibliography

\end{document}